\documentclass{article}
\usepackage[preprint]{neurips_2026}

\usepackage[utf8]{inputenc} 
\usepackage[T1]{fontenc}    
\usepackage{hyperref}       
\usepackage{url}            
\usepackage{booktabs}       
\usepackage{amsfonts}       
\usepackage{nicefrac}       
\usepackage{microtype}      
\usepackage{xcolor}         

\usepackage{amssymb}
\usepackage{amsthm}
\usepackage{mathtools}
\usepackage{xspace}

\usepackage{booktabs}
\usepackage{multirow}
\usepackage{adjustbox}
\usepackage{subcaption}

\usepackage{paralist}

\usepackage{enumitem}
\usepackage{tikz}

\usepackage[colorinlistoftodos]{todonotes}

\usetikzlibrary{arrows.meta,positioning,fit,calc}

\usepackage{listings}
\newcommand{\clients}{\mathcal{C}}
\newcommand{\Fed}{\mathsf{Fed}}
\newcommand{\RR}{\mathbb{R}}
\newcommand{\Sh}{\mathsf{Sh}}
\newcommand{\vglob}[1]{\overline{#1}}

\newtheorem{corollary}{Corollary}
\newtheorem{definition}{Definition}
\newtheorem{example}{Example}
\newtheorem{proposition}{Proposition}
\newtheorem{theorem}{Theorem}

\title{A Typed Tensor Language for Shared-State Federated Computation}

\newcommand{\authorname}[2]{%
  \textbf{#1}{\normalfont\textsuperscript{#2}}%
}

\newcommand{\affil}[2]{%
  {\normalfont\textsuperscript{#1}}{\normalfont #2}%
}

\author{
\authorname{Theofilos Mailis}{1,2}\thanks{Equal contribution}\quad
\authorname{Theodore Papamarkou}{1,3}\footnotemark[1]\quad
\authorname{Andreas Ktenidis}{1,4}\\[0.5ex]
\authorname{Kalliopi-Christina Despotidou}{1,2}\quad
\authorname{Konstantinos Filippopolitis}{1}\quad
\authorname{Yannis Foufoulas}{1}\\[0.5ex]
\authorname{Thanasis-Michail Karampatsis}{1}\quad
\authorname{Evdokia Mailli}{1}\quad
\authorname{Yannis Ioannidis}{1,2}\\[2ex]
\affil{1}{Athena Research Center}\quad
\affil{2}{National and Kapodistrian University of Athens}\\[0.5ex]
\affil{3}{National Technical University of Athens}\quad
\affil{4}{Athens University of Economics and Business}
}

\newcounter{replycount}
\newcommand{\circlednum}[1]{%
  \tikz[baseline=(char.base)]{
    \node[shape=circle, draw=green!40!black!100, fill=green!40!black!100, text=white,
          inner sep=0pt, minimum size=1.4em, font=\small\bfseries]
          (char) {#1};
  }%
}
\NewDocumentEnvironment{reply}{o}{%
  \stepcounter{replycount}%
  \par\noindent
  \color{green!20!black!100}%
  \circlednum{\arabic{replycount}}%
  \IfValueT{#1}{\textbf{\ #1}\textbf{:}}%
  \ %
}{%
  \par\vspace{0.8em}%
}

\begin{document}

\maketitle

\begin{abstract}
Shared-state federated computations combine client-local tensor computation, mergeable aggregation into shared state, and shared-only post-processing. We introduce a typed tensor language for this class of computations. Its two tensor sorts separate client-partitioned data from globally available values, and typing tracks the partitioned axis. A virtual global tensor serves as a semantic reference for centralized evaluation. We show that typed one-round programs factor through shared tensors whose shapes depend on the program but are independent of client and sample counts. The converse applies to typed-realizable factorizations: each encoder component is represented by an allowed aggregation or contraction with its valid merge, and the decoder is shared-only. The construction extends round by round to programs whose persistent state is shared. For a loss supplied with a client-local per-sample gradient expression, summation represents the empirical gradient. This gives typed programs for server-side first-order updates and, with shared linear algebra, curvature-block updates. The language covers federated analytics and FedSGD. General multi-local-step FedAvg and persistent private client state are outside its scope.
\end{abstract}

\section{Introduction}\label{sec:introduction}

Federated learning is usually introduced through a simple constraint: data are stored across clients, and raw records should not be centralized. The standard response is to let clients compute local quantities and let a server aggregate them into shared state. This pattern appears in gradient-based learning, second-order optimization, federated analytics, and classical statistical summaries such as means, covariances, normal-equation blocks, and curvature matrices.

Despite this common structure, federated computations are often specified algorithm by algorithm. A federated mean, gradient, covariance, optimizer update, or Newton block may be written as a separate protocol, even when they share the same mathematical form. This obscures a basic question: which tensor computations over distributed records can be expressed as client-local computation, followed by fixed-dimensional shared state, followed by shared-only post-processing?

This paper studies shared-state federated computation. Inter-client dependence passes through shared tensors whose shapes are independent of client and sample counts. The language agrees with evaluation on a virtual global tensor, used as a semantic reference. It covers means and covariances, FedSGD~\citep{mcmahan2017communication}, and server-side optimization. General multi-local-step FedAvg~\citep{mcmahan2017communication} and persistent private client state are outside the language's scope. Its two tensor sorts separate client-partitioned data from shared values, and typing tracks the axis along which the data are partitioned.

\textbf{Why record-axis typing matters.} Existing systems provide placement-aware programming interfaces for federated computation~\citep{beutel2020flower,tensorflowfederated2021}. Our focus is complementary; \emph{we track the tensor-level record axis and make clear when a computation remains client-local and when it produces shared state. This rules out expressions that treat a federated record axis as if it were already globally available, unless the computation passes through a designated record-axis elimination step.} For example, a normal-equation block such as $X^\top X$ is represented as $\sum_{c\in\clients}{X^{(c)}}^\top X^{(c)}$, not as an ordinary shared matrix product on a global tensor. \emph{The point is not to identify bugs in a particular framework, but to give formal tensor semantics in which locality, exposure, and centralized equivalence are visible in the type structure.}

\textbf{Shared-state factorization.} The central semantic object is a shared-state factorization. Let $\clients$ denote the set of clients, let $X=\{X^{(c)}\}_{c\in\clients}$ be a federated input, and let $n_c$ be the local number of records at client $c$. A federated computation admits a shared-state factorization if it can be written as $\sigma_{\clients}(X) = \psi\left( \bigoplus_{c\in\clients} \phi_{n_c}(X^{(c)})\right)$, where each encoder $\phi_{n_c}$ is evaluated locally on client $c$, $\oplus$ is a merge operation on fixed-dimensional client summaries, and $\psi$ is a shared-only decoder. Figure~\ref{fig:shared-state-factorization} visualizes this three-stage structure of local encoding, aggregation (merging) into shared state, and shared-only decoding. The shared state is finite-dimensional and fixed in advance; it may depend on the computation, but not on the number of clients or on their local record counts.

\begin{figure}[t]
\centering
\begin{tikzpicture}[
    >=Latex,
    font=\small,
    node distance=6mm and 8mm,
    box/.style={
        draw,
        rounded corners,
        align=center,
        inner sep=4pt,
        minimum height=8.5mm,
        text width=12mm
    },
    ebox/.style={
        draw,
        rounded corners,
        align=center,
        inner sep=4pt,
        minimum height=8.5mm,
        text width=25mm
    },
    sbox/.style={
        draw,
        rounded corners,
        align=center,
        inner sep=4pt,
        minimum height=9.5mm,
        minimum width=28mm
    },
    psibox/.style={
        draw,
        rounded corners,
        align=center,
        inner sep=4pt,
        minimum height=8.5mm,
        minimum width=15mm
    },
    every path/.style={draw, -Latex}
]

\node[box] (x1) {$X^{(c_1)}$};
\node[box, below=3mm of x1] (x2) {$X^{(c_2)}$};
\node[below=1mm of x2] (vdots) {$\vdots$};
\node[box, below=1mm of vdots] (xm) {$X^{(c_m)}$};

\node[ebox, right=10mm of x1] (e1) {$\phi_{n_{c_1}}\!\left(X^{(c_1)}\right)$};
\node[ebox, right=10mm of x2] (e2) {$\phi_{n_{c_2}}\!\left(X^{(c_2)}\right)$};
\node[ebox, right=10mm of xm] (em) {$\phi_{n_{c_m}}\!\left(X^{(c_m)}\right)$};

\node[sbox, right=12mm of e2, text width=31mm] (agg) {$\displaystyle s=\bigoplus_{c\in\clients}\phi_{n_c}(X^{(c)})$};

\node[psibox, right=12mm of agg] (dec) {$\psi(s)$};

\draw (x1) -- (e1);
\draw (x2) -- (e2);
\draw (xm) -- (em);

\draw (e1.east) -- ($(agg.west)+(0,4mm)$);
\draw (e2.east) -- (agg.west);
\draw (em.east) -- ($(agg.west)+(0,-4mm)$);

\draw (agg) -- (dec);

\node[above=1mm of x1, align=center] {\footnotesize Local input};
\node[above=1mm of e1, align=center] {\footnotesize Local encoders};
\node[above=1mm of agg, align=center] {\footnotesize Merge into shared state};
\node[above=1mm of dec, align=center] {\footnotesize Shared-only decoder};

\end{tikzpicture}
\caption{Shared-state factorization of a federated computation. Each client applies a local encoder to its local tensor, the encoded summaries are merged into a finite-dimensional shared state, and a shared-only decoder produces the final output.}
\label{fig:shared-state-factorization}
\end{figure}

\textbf{Contributions.} The three-step pattern of computing local summaries, aggregating them, and decoding the result has appeared in federated learning, database aggregation, and mergeable summaries~\citep{gray1997datacube,agarwal2012mergeable,mcmahan2017communication}. Our contribution is to give this pattern a typed federated tensor-language formulation, connect it to virtual-global semantics, and prove expressibility and factorization theorems for the shared-state program classes defined in Section~\ref{sec:shared-state-factorizations}. More specifically, our contributions are as follows.
\begin{enumerate}[leftmargin=*, topsep=0pt, itemsep=0pt]
\item \emph{Language definition and semantics}. We define a typed federated tensor language with federated and shared tensors, record-axis-aware typing, and virtual-global semantics. The language ensures client-locality, restricts how shared outputs may arise from federated inputs, and agrees with centralized evaluation on the virtual global tensor.
\item \emph{Factorization theorems and iteration}. We prove structural expressibility theorems for one-round and iterative programs. Each one-round program induces a shared-state factorization, and each typed-realizable factorization is represented by such a program. Iterative programs retain shared state between rounds.
\item \emph{Differentiable learning and optimization}. We specialize the framework to differentiable learning. Representably differentiable losses induce federated gradient tensors whose record-axis sums yield global gradients. This gives typed programs for server-side gradient descent, momentum, and Adam, and, after adding shared-only linear algebra primitives, curvature-block methods such as damped Newton, Gauss-Newton, and iteratively reweighted least squares (IRLS) whenever their updates are shared-only functions of aggregated gradients and curvature blocks.
\end{enumerate}

\textbf{Scope.}
Uniform client encoders and commutative merging exclude outputs that depend on client order. Representing median and quantile operators in our language requires a typed mergeable state of bounded dimension. Kernel methods retaining state whose dimension grows with sample counts fall outside our language's scope. A transient local-encoder interface expresses multi-step FedAvg/FedProx with round-wise factorization without virtual-global consistency. Persistent private state requires further semantics. Exposure discipline restricts shared-state formation and gives no privacy guarantee. Appendix~\ref{app:implementation} presents Federated Tensor Algebra~\citep[FeTA;][]{foufoulas2026}, which is our Python package implementation, and associated reported code-complexity results.

\textbf{Paper structure.} The rest of the paper is organized as follows. Section~\ref{sec:typed-language} introduces the typed federated tensor language and its semantic guarantees. Section~\ref{sec:shared-state-factorizations} proves the shared-state factorization results. Section~\ref{sec:differentiable-learning} specializes the theory to differentiable losses and learning updates. Section~\ref{sec:discussion} discusses limitations and extensions. Related work is discussed in Appendix~\ref{app:related-work}.

\section{Typed federated tensor language}
\label{sec:typed-language}

This section defines the base typed tensor language. The language separates federated tensors from shared tensors, tracks the record axis of federated data, and gives virtual-global semantics for comparing distributed execution with centralized tensor evaluation. We then state the three semantic guarantees used throughout the paper, namely client-locality, exposure discipline, and virtual-global consistency.

\subsection{Semantic objects}
\label{subsec:semantic-objects}

We work in a federated setting. Each client stores data that are partitioned along one tensor axis, called the \emph{record axis}. Shared tensors are available at every client. The \emph{virtual global tensor} is a centralized reference object obtained by concatenating client-local tensors along the record axis in a fixed client order. It is used only for semantics.

\begin{definition}[Federation]
\label{def:federation}
A federation is a finite nonempty set $\clients$ equipped with a fixed total order. We write $\clients = \{c_1,\dots,c_m\}$, $c_1 < \ldots < c_m$.
\end{definition}

For each integer $k \geq 1$, we write $[k] := \{1,\dots,k\}$ and $\mathbb N_0 := \{0,1,2,\dots\}$. We also set $[0]:=\varnothing$. Scalars are identified with $0$-way tensors.

\begin{definition}[Federated tensor]
\label{def:federated-tensor}
Let $k \geq 1$, let $r \in [k]$, and let $\mathbf d = (d_1,\dots,d_{r-1},d_{r+1},\dots,d_k)$ be a tuple of positive integers. When $k=1$, the tuple $\mathbf d$ is empty. A $k$-way federated tensor with record axis $r$ and common non-record shape $\mathbf d$ is a family $X = \{X^{(c)}\}_{c \in \clients}$ for which there exists a family of local record counts $\{n_c\}_{c \in \clients} \subseteq \mathbb N_0$ such that $X^{(c)} \in \RR^{d_1 \times \ldots \times d_{r-1} \times n_c \times d_{r+1} \times \ldots \times d_k}$ for every $c \in \clients$.
\end{definition}

\begin{definition}[Shared tensor]
\label{def:shared-tensor}
A shared tensor is either a scalar $S \in \RR$, or a tensor $S \in \RR^{s_1 \times \cdots \times s_k}$ for some integer $k \geq 1$ and some positive integers $s_1,\dots,s_k$, with the property that the same value of $S$ is available at every client.
\end{definition}

Global availability refers to the shared sort. A federated value retains its client-indexed sort even when there is one client.

\begin{definition}[Virtual global tensor]
\label{def:virtual-global-tensor}
Let $X = \{X^{(c)}\}_{c \in \clients}$ be a federated tensor with record axis $r$, common non-record shape $\mathbf d = (d_1,\dots,d_{r-1},d_{r+1},\dots,d_k)$, and local record counts $\{n_c\}_{c \in \clients}$. The virtual global tensor of $X$ is $\vglob{X} := \operatorname{Concat}_r(X^{(c_1)},\dots,X^{(c_m)})$, where $\clients = \{c_1,\dots,c_m\}$ is listed in increasing order and $\operatorname{Concat}_r$ denotes concatenation along axis $r$. Moreover,
\[
\vglob{X} \in \RR^{d_1 \times \ldots \times d_{r-1} \times \left(\sum_{j=1}^m n_{c_j}\right) \times d_{r+1} \times \ldots \times d_k}.
\]
\end{definition}

The fixed order on $\clients$ is used to define $\vglob{X}$. The tensor $\vglob{X}$ is a semantic reference object that is used to state the correctness of typed federated computations. The local record counts $\{n_c\}_{c \in \clients}$ are part of the value of a federated tensor.

\subsection{Typed expressions and primitive operations}
\label{subsec:typed-expressions-primitives}

The language has two tensor sorts. A shared type is written $\Sh(\mathbf s)$, where $\mathbf s$ is an ordinary shape tuple. A federated type is written $\Fed_r(\mathbf d)$, where $r$ is the record axis and $\mathbf d$ is the common non-record shape. This is the type-level counterpart of Definition~\ref{def:federated-tensor}.

Expressions are built from variables, corresponding to federated or shared tensors, by applying primitive tensor operations. A typing judgment has the form $\Gamma \vdash e : T$, where $\Gamma$ assigns tensor types to variables, $e$ is an expression, and $T$ is either shared or federated. The base primitive signature is denoted by $\Sigma_0$. It contains unary element-wise maps, binary element-wise maps, binary comparisons, axis-wise aggregations, axis permutations, and three matrix-product primitives. The full syntax, shape rules, broadcast rules, semantic clauses, and typing rules are given in Appendix~\ref{app:language-spec}.

The base signature uses scalar unary maps $\mathcal U$, binary maps $\mathcal B$, comparisons $\mathcal G$, aggregation families $\mathcal A$, and permutations $\mathcal P$. Figure~\ref{fig:aggregation-typing} gives the aggregation rules and the federated-federated matrix contraction. The other matrix primitives, $\mathsf{MatMul}_{\Fed\Sh}$ and $\mathsf{MatMul}_{\Sh\Fed}$, retain federated output. Full shape and compatibility conditions appear in Appendix~\ref{app:language-spec}.

\begin{figure}[t]
\centering
\[
\frac{\Gamma\vdash e:\Fed_r(\mathbf d)}
     {\Gamma\vdash\alpha_r(e):\Sh(\mathbf d)}
\qquad
\frac{\Gamma\vdash e:T\quad j\ne r}
     {\Gamma\vdash\alpha_j(e):\operatorname{del}_j(T)}
\]
\[
\frac{\Gamma\vdash a:\Fed_2((u))\quad
      \Gamma\vdash b:\Fed_1((v))}
     {\Gamma\vdash\mathsf{MatMul}_{\Fed\Fed}(a,b):\Sh((u,v))}
\]
\caption{Aggregation with $T=\Fed_r(\mathbf d)$, $\alpha\in\mathcal A$, and $j\in[|\mathbf d|+1]$. The first rule removes axis $r$, the second removes another axis and retains federated output. The contraction requires matching client-local lengths along the contracted axes. The factorization program class also requires the merge structure of Section~\ref{subsec:one-round-typed-programs}.}
\label{fig:aggregation-typing}
\end{figure}

\begin{example}[A federated mean]
For scalar samples, choose $\operatorname{Sum}\in\mathcal A$, the map $u_1(z)=1$ in $\mathcal U$, and the total division map below in $\mathcal B$:
\[
\phi_n(z)=\left(\sum_{a=1}^n z_a,n\right),~
\oplus=+,~
\psi(A,N)=\begin{cases}A/N,&N>0,\\0,&N\le0.\end{cases}
\]
Both components are typed sums, the second summing $u_1(z)$. An empty client sends $(0,0)$. For a positive total count the decoder returns the global mean; its other branch makes it total on $\RR^2$.

We use the following terminology throughout the paper. A \emph{typed primitive application} is an application of a primitive symbol in $\Sigma_0$ that satisfies the corresponding typing, shape, broadcast, and record-axis conditions of Appendix~\ref{app:language-spec}. A typed expression $\Gamma \vdash e : T$ is a \emph{shared expression} if $T$ is a shared type, and it is a \emph{federated expression} if $T$ is a federated type.
\end{example}

A federated expression is called \emph{client-local} if every primitive subexpression that receives at least one federated-typed argument has federated type. Thus, client-local expressions may use shared inputs, but they do not perform an intermediate operation that converts federated data into shared state. Record-axis aggregations and the federated-federated matrix product are therefore excluded from client-local expressions, except when they are the shared-output step being studied. A shared expression is called \emph{shared-only} if all free variables in its typing context have shared type. Shared-only expressions model post-processing performed after the shared state has been formed.

The main role of the type system is to separate local computation from shared-state formation. Client-local federated expressions preserve the record axis and can be evaluated independently on each client. Shared state is formed only when the record axis is eliminated by a designated shared-output primitive, such as record-axis aggregation or the federated-federated matrix product. Exposure discipline controls which operations form shared values, without providing differential privacy, cryptographic confidentiality, or protection against inference from those values. The next subsection records the three semantic guarantees that make this separation precise.

\subsection{Foundational semantic propositions}
\label{subsec:semantic-propositions}

The full formal specification of the base typed language is given in Appendix~\ref{app:language-spec}. In this subsection, a typed primitive application means an application of a primitive symbol of $\Sigma_0$ satisfying the typing, shape, broadcast, and axis conditions of that specification.

For each primitive symbol $F$ of $\Sigma_0$, we write $F^{\operatorname{ord}}$ for its ordinary centralized interpretation. If $F$ is one of the symbols in $\mathcal U \cup \mathcal B \cup \mathcal G \cup \{\alpha_j : \alpha \in \mathcal A,\ j \geq 1\} \cup \mathcal P$, then $F^{\operatorname{ord}}$ is the corresponding ordinary tensor operation. If $F$ is one of the three matrix-product symbols, then $F^{\operatorname{ord}}$ is ordinary matrix multiplication in the displayed operand order.

Proposition~\ref{prop:client-locality} states client-locality of every federated-output primitive application. Proposition~\ref{prop:exposure-discipline} identifies the primitive applications that can produce shared state from federated inputs. Proposition~\ref{prop:virtual-global-consistency} states agreement between the distributed primitive semantics and centralized evaluation on the virtual global tensor.

\begin{proposition}[Client-locality]
\label{prop:client-locality}
Let $F$ be a primitive symbol of $\Sigma_0$. Consider a typed primitive application of $F$ whose federated inputs are $X_1,\dots,X_p$ and whose shared inputs are $S_1,\dots,S_q$, and let $Y$ denote its output. Assume that $Y$ is a federated tensor. Then, for every client $c \in \clients$, there exists a deterministic map $\Phi_{F,c}$ such that $Y^{(c)} = \Phi_{F,c}(X_1^{(c)},\dots,X_p^{(c)},S_1,\dots,S_q)$. The maps $\Phi_{F,c}$ can be chosen so that they depend on the primitive symbol, the operand positions, and the local input shapes at client $c$, but not on the identity of $c$ or on any tensor stored at another client.
\end{proposition}

\begin{proposition}[Exposure discipline]
\label{prop:exposure-discipline}
Let $F$ be a primitive symbol of $\Sigma_0$. Consider a typed primitive application of $F$ whose federated inputs are $X_1,\dots,X_p$ and whose shared inputs are $S_1,\dots,S_q$, and let $Y$ denote its output. The output object is globally available if and only if $Y$ is a shared tensor. Assume that $p \geq 1$ and that $Y$ is shared. Then exactly one of the following two cases holds. There exist $\alpha \in \mathcal A$, an index $i \in [p]$, and the record axis $r$ of $X_i$ such that $Y = \alpha_r(X_i)$. There exist indices $i,j \in [p]$ with $i\ne j$ such that $Y = \mathsf{MatMul}_{\Fed\Fed}(X_i,X_j)$, where $\mathsf{MatMul}_{\Fed\Fed}$ is the federated-federated matrix-product primitive that contracts the record axes of two federated matrix inputs and returns a shared matrix (see Appendix~\ref{app:language-spec}).
\end{proposition}

\begin{proposition}[Virtual-global consistency]
\label{prop:virtual-global-consistency}
Let $F$ be a primitive symbol of $\Sigma_0$ with arity $m$. Consider a typed primitive application of $F$ to inputs $A_1,\dots,A_m$, where each $A_t$ is either a federated tensor or a shared tensor, and let $Y$ denote its output. For each $t \in [m]$, define
\[
\widetilde A_t
=
\begin{cases}
\vglob{A_t}, & \text{if } A_t \text{ is a federated tensor},\\
A_t, & \text{if } A_t \text{ is a shared tensor}.
\end{cases}
\]
If $Y$ is a federated tensor, then $\vglob{Y} = F^{\operatorname{ord}}(\widetilde A_1,\dots,\widetilde A_m)$. If $Y$ is a shared tensor, then $Y = F^{\operatorname{ord}}(\widetilde A_1,\dots,\widetilde A_m)$.
\end{proposition}

Propositions~\ref{prop:client-locality}, \ref{prop:exposure-discipline}, and~\ref{prop:virtual-global-consistency} complete the setup of the base typed language, and form the semantic input to the shared-state factorization theorems developed in Section~\ref{sec:shared-state-factorizations}. Their proofs are given in Appendix~\ref{app:semantic-proofs}.

\section{Shared-state factorizations}
\label{sec:shared-state-factorizations}

The semantic propositions of Section~\ref{sec:typed-language} show that the base language separates client-local computation from shared-state formation. We now use this separation to describe the class of shared-output computations represented by the language. The central object is a factorization through a fixed-dimensional shared state. The state is computed locally, merged across clients, and decoded into the final shared output. The proofs of the theoretical statements of this section are given in Appendix~\ref{app:shared-state-proofs}.

\subsection{Semantic shared-state factorizations}
\label{subsec:semantic-shared-state-factorizations}

Consider an input type $\Fed_r(\mathbf d)$, where $\mathbf d = (d_1,\dots,d_{r-1},d_{r+1},\dots,d_k)$, and an output shape $\mathbf t$. For each $n \in \mathbb N_0$, write $\mathcal X_n(r,\mathbf d) := \RR^{d_1 \times \ldots \times d_{r-1} \times n \times d_{r+1} \times \ldots \times d_k}$. A shared-output computation of input type $\Fed_r(\mathbf d)$ and output shape $\mathbf t$ is a rule $\sigma$ that assigns to every finite ordered federation $\clients$ and every federated tensor $X=\{X^{(c)}\}_{c \in \clients}$ of type $\Fed_r(\mathbf d)$ an output $\sigma_{\clients}(X) \in \RR^{\mathbf t}$.

\begin{definition}[Shared-state factorization]
\label{def:shared-state-factorization}
Let $\sigma$ be a shared-output computation of input type $\Fed_r(\mathbf d)$ and output shape $\mathbf t$. A shared-state factorization of $\sigma$ consists of a finite product of tensor spaces $M = \RR^{\mathbf s_1} \times \ldots \times \RR^{\mathbf s_q}$, where $q \geq 1$, a binary operation $\oplus_M : M \times M \to M$, an element $0_M \in M$, local encoder maps $\phi_n : \mathcal X_n(r,\mathbf d) \to M$ for every $n \in \mathbb N_0$, and a decoder map $\psi : M \to \RR^{\mathbf t}$, such that the following two conditions hold. First, $(M,\oplus_M,0_M)$ is a commutative monoid. Thus, for all $a,b,c \in M$, we have $(a \oplus_M b) \oplus_M c = a \oplus_M (b \oplus_M c)$, $a \oplus_M b = b \oplus_M a$, and $a \oplus_M 0_M = a$. Second, for every finite ordered federation $\clients$ and every federated tensor $X=\{X^{(c)}\}_{c \in \clients}$ of type $\Fed_r(\mathbf d)$ with local record counts $\{n_c\}_{c \in \clients}$, we have $\sigma_{\clients}(X) = \psi\left(\bigoplus_{c \in \clients} \phi_{n_c}(X^{(c)})\right)$.
\end{definition}

The finite product $M$ is independent of the federation and of the local record counts. This is the fixed-dimensional shared state of the computation. Definition~\ref{def:shared-state-factorization} is semantic and does not require the encoders or decoder to be expressed in the typed language. The next subsection connects this semantic factorization with typed one-round programs.

The same factorization also identifies the tensor-level interface for privacy mechanisms. Local differential privacy may be applied to the encoded client messages, secure multi-party computation may realize the merge, and central differential privacy may be applied to the merged state or decoded output. Appendix~\ref{app:privacy-integration} formalizes this observation as a privacy-lifting result for shared-state factorizations.

\subsection{One-round typed programs}
\label{subsec:one-round-typed-programs}

We restrict the semantic factorization of Definition~\ref{def:shared-state-factorization} to factorizations generated by typed expressions. A one-round program first evaluates finitely many client-local federated expressions. It then converts each of them into shared state by an allowed shared-output primitive. Finally, it applies a shared-only decoder. Throughout this subsection, expressions are over the base primitive signature $\Sigma_0$.

For an aggregation $\alpha_r$ with output shape $\mathbf s$, a \emph{merge structure} is a commutative monoid $(\RR^{\mathbf s},\oplus,0_{\alpha,\mathbf s})$ satisfying
\[
\alpha_r(U\mathbin{\|}_r V)=\alpha_r(U)\oplus\alpha_r(V),
~\alpha_r(E_{r,\mathbf s})=0_{\alpha,\mathbf s},
\]
for all tensors $U,V$ with non-$r$ shape $\mathbf s$, including empty blocks; $\|_r$ denotes concatenation on axis $r$, and $E_{r,\mathbf s}$ has this non-$r$ shape and extent zero on axis $r$. The identity can differ from the numerical zero tensor. Sum uses addition and zero. Product uses elementwise multiplication and the all-ones tensor, when included in the signature. Merge structures are required for aggregation boundaries in the one-round and iterative typed program classes (Definitions~\ref{def:one-round-typed-program} and~\ref{def:iterative-typed-program}).

A shared-state factorization is \emph{typed-realizable} when its component encoders have the valid aggregation or aligned contraction forms used below, and its decoder has a shared-only realization on the whole state space. Definition~\ref{def:typed-realizable-factorization} gives the realization witnesses.

\begin{definition}[One-round typed program]
\label{def:one-round-typed-program}
Consider an input type $\Fed_r(\mathbf d)$ and an output shape $\mathbf t$. A one-round typed program consists of shared-state expressions $g_1(x),\dots,g_q(x)$ and a shared-only decoder expression $h$, where $q \geq 1$, with the following properties.

For each $i \in [q]$, the expression $g_i(x)$ has one of the following two forms. The first form is record-axis aggregation. There are a client-local federated expression $e_i(x)$, a type $\Fed_{r_i}(\mathbf u_i)$, and an aggregation schema $\alpha^{(i)} \in \mathcal A$ such that $x : \Fed_r(\mathbf d) \vdash e_i(x) : \Fed_{r_i}(\mathbf u_i)$ and $g_i(x) = \alpha^{(i)}_{r_i}(e_i(x))$. The output shape of $g_i(x)$ is denoted by $\mathbf s_i$. The aggregation has a merge structure $(\RR^{\mathbf s_i},\oplus_i,0_i)$. Its concatenation law gives $\alpha^{(i)}_{r_i}(Z) = \bigoplus_{c \in \clients} \alpha^{(i)}_{r_i}(Z^{(c)})$ for every federated tensor $Z$ of the indicated type. The second form is a federated-federated matrix product. There are client-local federated expressions $a_i(x)$ and $b_i(x)$ and positive integers $m_i,n_i$ such that $x : \Fed_r(\mathbf d) \vdash a_i(x) : \Fed_2((m_i))$ and $x : \Fed_r(\mathbf d) \vdash b_i(x) : \Fed_1((n_i))$, and $g_i(x)=\mathsf{MatMul}_{\Fed\Fed}(a_i(x),b_i(x))$. The two operands must satisfy the client-wise count compatibility condition of Definition~\ref{def:fed-fed-matmul-semantics}. In this case, $\mathbf s_i=(m_i,n_i)$ and the associated merge operation on $\RR^{\mathbf s_i}$ is ordinary addition.

The decoder is a shared-only expression $y_1:\Sh(\mathbf s_1),\dots,y_q:\Sh(\mathbf s_q) \vdash h(y_1,\dots,y_q):\Sh(\mathbf t)$. The shared-output computation induced by the program is $\sigma_{\clients}(X) = h(g_1(X),\dots,g_q(X))$ for every finite ordered federation $\clients$ and every federated input $X$ of type $\Fed_r(\mathbf d)$.
\end{definition}

The two forms of $g_i$ are the shared-output mechanisms from federated inputs in Proposition~\ref{prop:exposure-discipline}. The merge structure restricts aggregation boundaries to client-wise merges whose output shape is independent of client and sample counts. The next result gives a semantic shared-state factorization for each one-round typed program.

\begin{theorem}[One-round shared-state factorization]
\label{thm:one-round-shared-state-factorization}
Let $\sigma$ be the shared-output computation induced by a one-round typed program of input type $\Fed_r(\mathbf d)$ and output shape $\mathbf t$. Then $\sigma$ admits a shared-state factorization (see Definition~\ref{def:shared-state-factorization}).
\end{theorem}

The construction uses structural induction for uniform local realizations (Proposition~\ref{prop:client-locality}), the permitted shared-output boundaries (Proposition~\ref{prop:exposure-discipline}), and their product merge. Shared-only decoding completes the factorization; Proposition~\ref{prop:virtual-global-consistency} relates base execution to ordinary tensor evaluation. Iteration repeats the construction with the current shared parameters.

\begin{proposition}[Effective factorization extraction]
\label{prop:effective-factorization-extraction}
From a valid merge-annotated derivation $\mathcal D$ of a one-round typed program, the factorization of Theorem~\ref{thm:one-round-shared-state-factorization} can be extracted in $O(|\mathcal D|)$ time and represented in $O(|\mathcal D|)$ space, with shared-subexpression DAG output and constant-time primitive metadata lookup.
\end{proposition}

\begin{corollary}[Extracted state dimension]
\label{cor:extracted-state-dimension}
For $\mathbf s=(s_1,\ldots,s_k)$, let $\nu(\mathbf s)=\prod_{j=1}^k s_j$, with $\nu(())=1$. The product state extracted from a one-round program $P$ has dimension
\[
D(P):=\dim M=\sum_{i=1}^q\nu(\mathbf s_i).
\]
\end{corollary}

Each declared shared-state component contributes to the sum appearing in $D(P)$, including scalars and repeated components. The quantity $D(P)$ is a constructive upper bound for representing the computation in real coordinates, without a claim of minimum dimension or communication optimality. Bounds on communication in bits and statistical accuracy, such as those of \citet{braverman2016distributed,zhang2013distributed}, use task and protocol assumptions. Comparison requires an accuracy criterion and a finite-precision protocol model, which we leave for future work (Appendix~\ref{app:complexity}).

The converse concerns the typed-realizable class of Definition~\ref{def:typed-realizable-factorization}.

\begin{theorem}[One-round realization of shared-state factorizations]
\label{thm:one-round-realization}
Every typed-realizable shared-state factorization is induced by a one-round typed program.
\end{theorem}

The two theorems characterize the one-round class with the specified encoder, merge and decoder forms. The class is closed under shared-only post-processing, as used for means and variances (Appendix~\ref{app:shared-postprocessing}). It covers typed realizations of the distributive and algebraic aggregates of~\citet{gray1997datacube}. Representing median or quantile operators in our typed tensor language requires a typed mergeable sufficient state with shape independent of client and sample counts. Finitely many real coordinates, without restrictions on their encoding, do not establish an impossibility result for such operators.

\subsection{Iterative typed programs}
\label{subsec:iterative-typed-programs}

Learning algorithms usually produce a sequence of shared states rather than a single summary. The shared state may be a parameter vector, a collection of optimizer moments, a curvature approximation, or a nonempty finite tuple of shared tensors. The same factorization principle applies round by round. Each round uses the current shared state inside client-local expressions, forms a finite shared state by record-axis elimination, and then applies a shared-only update map.

A tuple with component shapes $\mathbf s_1,\ldots,\mathbf s_q$ can be stored as a shared vector of length $\sum_{j=1}^q\nu(\mathbf s_j)$ by concatenating flattened components. Packing and component extraction are total shared-only primitives of a conservative extension (Appendix~\ref{app:conservative-extensions}). Extraction selects the component's block and restores its shape.

\begin{definition}[Iterative typed program]
\label{def:iterative-typed-program}
Fix an input type $\Fed_r(\mathbf d)$, an integer $T \geq 1$, and shared-state shapes $\boldsymbol{\tau}_0,\dots,\boldsymbol{\tau}_T$. An iterative typed program consists of an initial shared state $\theta_0 \in \RR^{\boldsymbol{\tau}_0}$ and, for each round $t \in \{0,\dots,T-1\}$, shared-state expressions $g_{t,1}(x,\theta),\dots,g_{t,q_t}(x,\theta)$ and a shared-only decoder expression $h_t$, where $q_t \geq 1$, with the following properties.

For each $i \in [q_t]$, the expression $g_{t,i}(x,\theta)$ has one of the two forms allowed in Definition~\ref{def:one-round-typed-program}, with the typing context enlarged by the current shared state $\theta:\Sh(\boldsymbol{\tau}_t)$. In the aggregation form, there are a client-local federated expression $e_{t,i}(x,\theta)$, a type $\Fed_{r_{t,i}}(\mathbf u_{t,i})$, and an aggregation schema $\alpha^{(t,i)} \in \mathcal A$ such that $x:\Fed_r(\mathbf d),\theta:\Sh(\boldsymbol{\tau}_t) \vdash e_{t,i}(x,\theta):\Fed_{r_{t,i}}(\mathbf u_{t,i})$ and $g_{t,i}(x,\theta)=\alpha^{(t,i)}_{r_{t,i}}(e_{t,i}(x,\theta))$. The output shape of $g_{t,i}$ is denoted by $\mathbf s_{t,i}$. The aggregation has a merge structure $(\RR^{\mathbf s_{t,i}},\oplus_{t,i},0_{t,i})$, so that $\alpha^{(t,i)}_{r_{t,i}}(Z) = \bigoplus_{c \in \clients} \alpha^{(t,i)}_{r_{t,i}}(Z^{(c)})$ for every federated tensor $Z$ of the indicated type. In the matrix-product form, there are client-local federated expressions $a_{t,i}(x,\theta)$ and $b_{t,i}(x,\theta)$ and positive integers $m_{t,i},n_{t,i}$ such that $x:\Fed_r(\mathbf d),\theta:\Sh(\boldsymbol{\tau}_t) \vdash a_{t,i}(x,\theta):\Fed_2((m_{t,i}))$ and $x:\Fed_r(\mathbf d),\theta:\Sh(\boldsymbol{\tau}_t) \vdash b_{t,i}(x,\theta):\Fed_1((n_{t,i}))$. In this case, $g_{t,i}(x,\theta)=\mathsf{MatMul}_{\Fed\Fed}(a_{t,i}(x,\theta),b_{t,i}(x,\theta))$, $\mathbf s_{t,i}=(m_{t,i},n_{t,i})$, and the associated merge operation on $\RR^{\mathbf s_{t,i}}$ is ordinary addition. The contraction satisfies the same client-wise count compatibility condition as in Definition~\ref{def:one-round-typed-program}.

The round-$t$ decoder is a shared-only expression $y_1:\Sh(\mathbf s_{t,1}),\dots,y_{q_t}:\Sh(\mathbf s_{t,q_t}),\theta:\Sh(\boldsymbol{\tau}_t) \vdash h_t(y_1,\dots,y_{q_t},\theta):\Sh(\boldsymbol{\tau}_{t+1})$.

For a finite ordered federation $\clients$ and a federated input $X$ of type $\Fed_r(\mathbf d)$, the program induces shared iterates by $\theta_{t+1} = h_t(g_{t,1}(X,\theta_t),\dots,g_{t,q_t}(X,\theta_t),\theta_t)$ for $t=0,\dots,T-1$. The induced shared-output computation is $\sigma_{\clients}(X):=\theta_T$.
\end{definition}

Definition~\ref{def:iterative-typed-program} does not introduce private state that persists across rounds. The only state carried from one round to the next is the shared state $\theta_t$. This matches algorithms whose communication pattern is server-side state updated from client-local tensor computations and shared aggregation.

\begin{theorem}[Iterative shared-state factorization]
\label{thm:iterative-shared-state-factorization}
Let $\sigma$ be induced by an iterative typed program of input type $\Fed_r(\mathbf d)$ and shared-state shapes $\boldsymbol{\tau}_0,\dots,\boldsymbol{\tau}_T$. For each round $t \in \{0,\dots,T-1\}$, define the round update $U_{t,\clients}(X,\theta) := h_t(g_{t,1}(X,\theta),\dots,g_{t,q_t}(X,\theta),\theta)$. Then, for each $t \in \{0,\dots,T-1\}$, there exist a finite product of tensor spaces $M_t$, a commutative monoid operation $\oplus_t$ on $M_t$, an identity element $0_t \in M_t$, local encoder maps $\phi_{t,n}(\cdot;\theta):\mathcal X_n(r,\mathbf d)\to M_t$ for every $n \in \mathbb N_0$ and every $\theta \in \RR^{\boldsymbol{\tau}_t}$, and a decoder map $\psi_t:M_t \times \RR^{\boldsymbol{\tau}_t}\to \RR^{\boldsymbol{\tau}_{t+1}}$ such that $U_{t,\clients}(X,\theta)= \psi_t\!\left(\bigoplus_{c\in\clients}\phi_{t,n_c}(X^{(c)};\theta),\theta\right)$ for every finite ordered federation $\clients$, every federated tensor $X=\{X^{(c)}\}_{c\in\clients}$ of type $\Fed_r(\mathbf d)$ with local record counts $\{n_c\}_{c\in\clients}$, and every shared state $\theta \in \RR^{\boldsymbol{\tau}_t}$.
\end{theorem}

Theorem~\ref{thm:iterative-shared-state-factorization} is the round-wise analogue of Theorem~\ref{thm:one-round-shared-state-factorization}. It shows that iterative typed programs do not introduce a new communication principle. Each round is still an encode, merge, and decode computation through fixed-dimensional shared state.

For the one-round and iterative program classes, the language and its conservative extensions retain uniform locality, the designated shared-output boundaries, factorization and virtual-global consistency, with only shared state across rounds (Appendix~\ref{app:round-wise-shared-memory}). A separate transient-state interface allows local optimizer steps and messages of count-independent shape, retaining factorization and exposure through the designated merge but producing partition-dependent updates; this covers multi-step FedAvg and FedProx (Appendix~\ref{app:transient-intra-round-state}). Persistent private client variables, as in SCAFFOLD~\citep{karimireddy20a} or FedDyn~\citep{acar2021federated}, require a further state sort and cross-round semantics beyond this paper.

\section{Differentiable fragment and learning updates}
\label{sec:differentiable-learning}

The factorization results of Section~\ref{sec:shared-state-factorizations} apply to the one-round and iterative typed program classes. We now specialize them to learning updates. The basic case is an empirical objective whose per-record loss is computed locally and whose gradient is obtained by summing per-record gradients across clients. This gives a direct route from the typed language to first-order federated optimization. The proofs of the theoretical statements of this section are given in Appendix~\ref{app:learning-update-proofs}.

Throughout this section, assume that $\mathcal A$ contains ordinary summation, denoted by $\operatorname{Sum}$. If $\boldsymbol{\tau}$ is a shape tuple, we identify $\RR^{\boldsymbol{\tau}}$ with a Euclidean space using the standard coordinate inner product. For a differentiable map $f:\RR^{\boldsymbol{\tau}}\to\RR$, we write $\nabla_\theta f(\theta) \in \RR^{\boldsymbol{\tau}}$ for its gradient.

\subsection{Representably differentiable losses and gradient expressibility}
\label{subsec:representably-differentiable-losses}

Let $\Fed_r(\mathbf d)$ and $\Sh(\boldsymbol{\tau})$ be an input type and a shared parameter type, respectively. A per-record scalar loss is represented in the language by a client-local federated expression $x:\Fed_r(\mathbf d),\theta:\Sh(\boldsymbol{\tau}) \vdash \ell(x,\theta):\Fed_1(\varnothing)$. Thus, for a client with $n$ records, the local value of $\ell$ is a vector in $\RR^n$. A per-record parameter-gradient expression has type $x:\Fed_r(\mathbf d),\theta:\Sh(\boldsymbol{\tau}) \vdash \gamma(x,\theta):\Fed_1(\boldsymbol{\tau})$. For a client with $n$ records, the local value of $\gamma$ is a tensor in $\RR^{n \times \boldsymbol{\tau}}$, with one gradient tensor of shape $\boldsymbol{\tau}$ for each record.

\begin{definition}[Representably differentiable loss]
\label{def:representably-differentiable-loss}
A pair of client-local federated expressions $(\ell,\gamma)$ with the typings above is called a representably differentiable loss if the following condition holds. For every $n \in \mathbb N_0$, every local input $Z \in \mathcal X_n(r,\mathbf d)$, and every $\theta \in \RR^{\boldsymbol{\tau}}$, let $\widehat \ell_n(Z;\theta) \in \RR^n$ and $\widehat \gamma_n(Z;\theta) \in \RR^{n \times \boldsymbol{\tau}}$ denote the local realizations of $\ell$ and $\gamma$. Then, if $n \geq 1$, for each record index $a \in [n]$,the map $\theta \mapsto \widehat \ell_n(Z;\theta)[a]$ is differentiable with $\widehat \gamma_n(Z;\theta)[a,\cdot]=\nabla_\theta \widehat \ell_n(Z;\theta)[a]$.
\end{definition}

Definition~\ref{def:representably-differentiable-loss} separates two issues. The expression $\ell$ computes the per-record loss, while $\gamma$ represents the corresponding per-record gradient as an expression in the typed language. The gradient expression $\gamma$ is provided with the loss; the present theory does not construct it by automatic differentiation. Theorem~\ref{thm:gradient-expressibility} only needs the existence of such an expression.

Let $(\ell,\gamma)$ be a representably differentiable loss of input type $\Fed_r(\mathbf d)$ and parameter type $\Sh(\boldsymbol{\tau})$. For a finite ordered federation $\clients$, a federated input $X=\{X^{(c)}\}_{c\in\clients}$, and a shared parameter  $\theta \in \RR^{\boldsymbol{\tau}}$, define the empirical objective  $L_{\clients}(X,\theta):=\sum_{c\in\clients}\sum_{a=1}^{n_c} \widehat \ell_{n_c}(X^{(c)};\theta)[a]$ and $G_{\clients}(X,\theta):=\sum_{c\in\clients}\sum_{a=1}^{n_c} \widehat \gamma_{n_c}(X^{(c)};\theta)[a,\cdot]\in\RR^{\boldsymbol{\tau}}$, where $n_c$ is the local record count.

\begin{theorem}[Gradient expressibility]
\label{thm:gradient-expressibility}
Let $(\ell,\gamma)$ be a representably differentiable loss. For every finite ordered federation $\clients$ and every federated input $X$ of type $\Fed_r(\mathbf d)$, the map $\theta \mapsto L_{\clients}(X,\theta)$ is differentiable and $\nabla_\theta L_{\clients}(X,\theta) = G_{\clients}(X,\theta)$. Moreover, $G_{\clients}(X,\theta)$ is represented by the shared-state expression $g(x,\theta) := \operatorname{Sum}_1(\gamma(x,\theta))$, where the summation is along the record axis of $\gamma(x,\theta)$. Thus, $x:\Fed_r(\mathbf d),\theta:\Sh(\boldsymbol{\tau}) \vdash g(x,\theta):\Sh(\boldsymbol{\tau})$, and the gradient computation is a round-wise shared-state computation with ordinary addition as its merge operation.
\end{theorem}

The theorem gives the basic learning interface of the language. Once the per-record gradient is expressible as a client-local federated tensor, the global gradient is obtained by a single record-axis summation. No additional federated primitive is needed.

\subsection{First-order server-side updates}
\label{subsec:first-order-updates}

The gradient expression of Theorem~\ref{thm:gradient-expressibility} can be used inside any shared-only server update. This includes gradient descent as the simplest case, and optimizers whose auxiliary variables are shared tensors.

\begin{corollary}[Server-side first-order updates]
\label{cor:server-side-first-order-updates}
Let $(\ell,\gamma)$ be a representably differentiable loss with parameter type $\Sh(\boldsymbol{\tau})$. Fix $T \geq 1$, shared-state shapes $\boldsymbol{\zeta}_0,\dots,\boldsymbol{\zeta}_T$, and an initial shared optimizer state $s_0 \in \RR^{\boldsymbol{\zeta}_0}$. For each round $t \in \{0,\dots,T-1\}$, assume that there are shared-only expressions $s:\Sh(\boldsymbol{\zeta}_t) \vdash P_t(s):\Sh(\boldsymbol{\tau})$ and $g:\Sh(\boldsymbol{\tau}),s:\Sh(\boldsymbol{\zeta}_t) \vdash H_t(g,s):\Sh(\boldsymbol{\zeta}_{t+1})$. For a finite ordered federation $\clients$ and a federated input $X$ of type $\Fed_r(\mathbf d)$, define the iterates by $s_{t+1} = H_t(G_{\clients}(X,P_t(s_t)),s_t)$, $t=0,\dots,T-1$. Then the $T$-step update $X \mapsto s_T$ is induced by an iterative typed program. Hence each round admits a shared-state factorization in the sense of Theorem~\ref{thm:iterative-shared-state-factorization}.
\end{corollary}

For gradient descent, the optimizer state is the parameter itself, $P_t$ is the identity, and $H_t(g,\theta)=\theta-\eta_t g$. Server-side momentum and Adam are covered whenever the optimizer state is shared and the parameter extraction and update maps are expressible as shared-only tensor expressions. In those cases, the shared optimizer state contains the parameter and the required moment tensors, and the only federated quantity used by the optimizer is the aggregated gradient. For a parameter with $p$ coordinates, this summed-gradient message has $D=p$; momentum and Adam retain $O(p)$ shared server memory. Appendix~\ref{app:complexity} separates message size from evaluation cost.

\subsection{Shared linear algebra and second-order updates}
\label{subsec:second-order-updates}

First-order methods require only an aggregated gradient. Second-order and quasi-second-order methods require additional shared tensors, such as Hessian blocks, Gauss--Newton blocks, or normal-equation matrices. These methods remain within the same framework when the curvature blocks are expressible as record-axis sums of client-local tensor expressions and the linear algebra is performed on shared tensors.

In this subsection, the language may be enlarged by shared-only linear-algebra primitives, such as matrix multiplication, linear-system solve, Cholesky factorization, or log-determinant, on their natural domains. Since these primitives act only on shared tensors, they do not change the client-locality or exposure properties of the federated part of the language. The diagonal shift $C+\lambda I$ denotes a shared-only primitive $\operatorname{AddDiag}_{\lambda}(C)$, whose ordinary interpretation adds the chosen scalar $\lambda$ to each diagonal entry. For operations with restricted domains, the factorization uses a total shared extension; the update is recovered on inputs whose required applications remain in the original domain (Appendix~\ref{app:conservative-extensions}).

\begin{corollary}[Curvature-block updates]
\label{cor:curvature-block-updates}
Fix a parameter dimension $p \geq 1$. Let $(\ell,\gamma)$ be a representably differentiable loss with parameter type $\Sh((p))$. Let $x:\Fed_r(\mathbf d),\theta:\Sh((p)) \vdash B(x,\theta):\Fed_1((p,p))$ be a client-local federated expression. For a finite ordered federation $\clients$, a federated input $X$, and a shared parameter $\theta \in \RR^p$, define $C_{\clients}(X,\theta) := \sum_{c\in\clients}\sum_{a=1}^{n_c} \widehat B_{n_c}(X^{(c)};\theta)[a,\cdot,\cdot] \in \RR^{p\times p}$, where $\widehat B_n$ denotes the local realization of $B$.

Fix $T \geq 1$, shared-state shapes $\boldsymbol{\zeta}_0,\dots,\boldsymbol{\zeta}_T$, and an initial shared optimizer state $s_0 \in \RR^{\boldsymbol{\zeta}_0}$. For each round $t \in \{0,\dots,T-1\}$, assume that there are shared-only expressions $s:\Sh(\boldsymbol{\zeta}_t) \vdash P_t(s):\Sh((p))$ and $g:\Sh((p)),C:\Sh((p,p)),s:\Sh(\boldsymbol{\zeta}_t) \vdash H_t(g,C,s):\Sh(\boldsymbol{\zeta}_{t+1})$. Define the iterates by $s_{t+1} = H_t(G_{\clients}(X,P_t(s_t)),C_{\clients}(X,P_t(s_t)),s_t),~t=0,\dots,T-1$. Then the $T$-step update $X \mapsto s_T$ is induced by an iterative typed program in the shared-linear-algebra extension of the language. Hence each round admits a shared-state factorization in the sense of Theorem~\ref{thm:iterative-shared-state-factorization}.
\end{corollary}

Newton, damped Newton, Gauss--Newton, and IRLS-type updates are instances of Corollary~\ref{cor:curvature-block-updates} when their curvature or normal-equation blocks are expressible by $B$. For example, a damped Newton step has shared-only decoder $H_t(g,C,\theta)=\theta-\eta_t\operatorname{Solve}(C+\lambda_t I,g)$ on the domain where $C+\lambda_t I$ is nonsingular. Appendix~\ref{app:learning-examples} gives two learning examples. The dense message $(g,C)$ has $D=p+p^2$. Structured curvature uses smaller messages when represented by suitable mergeable expressions. Otherwise it changes or approximates the update.

The base signature is intentionally small, but applications may add conservative primitives that are either client-local or shared-only, such as fixed projections, reshaping, model-specific differentiable maps, shared matrix multiplication, or shared linear solves. These extensions require uniform client-local realizations, agreement with ordinary evaluation on virtual global tensors, and no new federated-to-shared boundary. A formal statement is given as Proposition~\ref{prop:conservative-signature-extensions} in Appendix~\ref{app:conservative-extensions}.

\section{Discussion}
\label{sec:discussion}

The introduced typed tensor language describes shared-state federated computation with virtual-global consistency. Local training can retain round-wise factorization while depending on the client partition. Exposure discipline gives no differential privacy or cryptographic confidentiality and does not bound inference from a shared statistic. Appendix~\ref{app:privacy-integration} gives privacy lifting under its adjacency, sensitivity, and randomizer assumptions. Along these lines, future work will develop the privacy-preserving capabilities of our typed tensor language.

The learning results are expressibility results, not new optimizer-convergence theorems. When a represented federated update agrees with the corresponding centralized update on the virtual global tensor, any applicable centralized convergence theorem transfers under its own assumptions. The theory does not analyze statistical heterogeneity, client sampling, compression, stochastic approximation error, or optimizer-specific rates. Future work includes richer federated tensor languages and execution/backend optimizations for distributed tensor programs.

\begin{ack}
This project/research has received funding from the European Union's Research and Innovation Program Horizon Europe under Grant Agreement no. 101147319 (EBRAINS 2.0). We gratefully acknowledge NVIDIA Corporation for supporting this research with the donation of four NVIDIA RTX PRO 6000 Blackwell GPUs through the NVIDIA Academic Grant Program. We also thank PolyShape P.C., Greece, for covering the OpenAI usage costs incurred during this research.
%
%
\end{ack}










\appendix

\section{Related work}
\label{app:related-work}

This appendix gives an extended discussion of related work. The main paper focuses on the typed tensor language and its factorization theorems. Here we place the contribution in relation to federated learning and analytics, federated systems, declarative machine learning systems, tensor languages, differentiable distributed programming, and privacy mechanisms.

\subsection{Federated learning, analytics, and computation}
\label{app:rw-fl-fa}

Federated learning provides the main context for this work. It studies model training from decentralized data while keeping raw records on client devices or within data-holding institutions~\citep{mcmahan2017communication,kairouz2021advances}. The standard federated learning pattern is that clients compute local updates and a coordinating server aggregates those updates into a shared model. This pattern led to a large literature on optimization, systems, privacy, heterogeneity, and deployment~\citep{bonawitz2019towards,li2020federated}. Algorithms such as FedAvg~\citep{mcmahan2017communication}, FedProx~\citep{li2020federated}, and FedBN~\citep{li2021fedbn} are important examples of this model-training view.

Federated analytics studies aggregate analysis over decentralized data without centralizing raw records~\citep{wang2022federatedanalytics}. Its goals overlap with federated learning but are not identical. Analytics tasks often compute summaries, histograms, distributional quantities, or other population-level statistics rather than model updates. Recent work on federated data distribution shift estimation gives one example of analytics-oriented federated computation beyond model training~\citep{cormode2025federatedshift}. More broadly, federated computation can be viewed as a larger class of distributed computations in which data stay local and only summary information is shared~\citep{bharadwaj2022federatedtutorial}.

Our work is closest to this broad federated computation view. The difference is that we do not begin from a particular deployed system, query interface, or optimizer. We define a typed tensor language and prove shared-state factorization results for the computations it expresses. This clarifies the distinction between client-local tensor computation, record-axis elimination into shared state, and shared-only post-processing.

Federated technologies are also motivated by privacy, governance, and institutional data-sharing constraints. These concerns arise in settings such as healthcare, finance, transportation, wireless networks, and the Internet of Things~\citep{rieke2020future,long2020federated,pandya2023federated,niknam2020federated,nguyen2021federated}. Legal and ethical discussions of data use, including GDPR and CCPA-oriented treatments, provide part of the broader context~\citep{floridi2016data,zuboff2023age,voigt2017eu,bonta2022california}. Our paper does not claim to solve legal compliance or privacy by itself. It gives an exposure discipline at the level of typed tensor semantics. Cryptographic, differential-privacy, or trusted-execution mechanisms remain separate layers.

\subsection{Federated systems and TensorFlow Federated}
\label{app:rw-federated-systems}

Many federated learning systems provide APIs for building and deploying federated algorithms. TensorFlow Federated includes a Federated Core programming environment for distributed computations and a higher-level federated learning layer~\citep{tensorflowfederated2021,tff2019}. Other systems, including OpenFL, Flower, PySyft, NVIDIA FLARE, and Sequoia, provide different abstractions for orchestration, client execution, secure computation, research experimentation, or distributed privacy-preserving workflows~\citep{foley2022openfl,beutel2020flower,ziller2021pysyft,roth2022nvidia,roth2025empowering,xu2025sequoia}.

TensorFlow Federated (TFF) is the closest systems comparison because its Federated Core also provides a typed representation of federated computations. Our work is complementary rather than competing (see Table~\ref{tab:tff-comparison} for a comparison). TFF is an implementation framework for writing and executing federated computations. Our paper abstracts away runtime execution and gives tensor-level semantics that tracks the record axis, distinguishes federated tensors from shared tensors, and proves shared-state factorization theorems.

\begin{table}[t]
\centering
\caption{Comparison with TensorFlow Federated. The comparison is intended only to position the mathematical contribution, not to evaluate systems.}
\label{tab:tff-comparison}
\begin{tabular}{p{0.20\linewidth}p{0.36\linewidth}p{0.36\linewidth}}
\toprule
Aspect & TensorFlow Federated & Our typed tensor language \\
\midrule
Primary role & Programming and execution framework for federated computations & Mathematical typed tensor language and semantic theory \\
Typing focus & Federated placements and structured values & Federated/shared tensor sorts, tracked record axis, and shared-state formation \\
Core semantics & Operational programming model for distributed computations & Virtual-global tensor semantics for correctness against centralized evaluation \\
Exposure boundary & Controlled by framework-level federated operators and placements & Formal exposure discipline through tensor types and record-axis elimination \\
Main theorem & Not primarily a theorem-driven formalization & Client-locality, virtual-global consistency, and shared-state factorization theorems \\
\bottomrule
\end{tabular}
\end{table}

The distinction matters for the kind of results proved in our paper. A placement-aware framework can ensure that values are located at clients or at the server. Our type system refines this with tensor-level record-axis information. This lets us state, for example, that an expression treating a federated record axis as a shared tensor axis is invalid unless it passes through an explicit record-axis elimination. The result is a formal account of what the computation means and how it factors through shared state, rather than a proposal for how a specific runtime should schedule the computation.

Appendix~\ref{app:implementation} compares a covariance computation in FeTA's Python implementation with TFF orchestration of the same computation. FeTA infers the local and shared stages from operations on the data axis; both approaches retain the communication dependencies of the algorithm.

\subsection{Declarative, algebraic, and database systems}
\label{app:rw-declarative-algebraic}

A related line of work develops declarative or algebraic abstractions for machine learning systems. Tensor relational algebra gives an algebraic backend for distributed machine learning system design based on tensor relations and execution plans~\citep{yuan2021tensor}. SystemDS develops a declarative system for end-to-end data science pipelines and supports local, distributed, and federated execution~\citep{boehm2020systemds}. These works share with ours the view that structured mathematical computation should not be treated as a collection of opaque programs.

Declarative systems and algebraic backends are primarily concerned with specifying and optimizing computations across execution modes. Our paper is concerned with typed semantics for federated tensor computation. In particular, our language makes the transition from client-local state to shared state formal through the record axis. The shared-state factorization theorems then characterize computations that pass through fixed-dimensional shared state.

Classical database work on aggregation is also relevant. The data cube literature distinguishes distributive and algebraic aggregate functions from holistic aggregates requiring unbounded intermediate state~\citep{gray1997datacube}. Mergeable summaries give another formal account of summaries that can be combined from independently computed pieces~\citep{agarwal2012mergeable}. Our shared-state factorizations are close in spirit to these ideas, but they are formulated for typed tensor expressions over federated records. They also connect the same finite-state principle to iterative learning updates, gradients, and curvature blocks.

Recent database work also studies user-defined functions from the perspective of query optimization. For example, GRACEFUL learns cost estimates for SQL query plans containing UDFs~\citep{wehrstein2025graceful}. This line is orthogonal to our focus. GRACEFUL studies execution-cost prediction for database query optimization, while our paper studies typed semantics and shared-state factorization for federated tensor computations.

\subsection{Tensor languages, tensor programs, and differentiable distributed programming}
\label{app:rw-tensor-languages}

Tensor and array languages study how to express structured tensor computation with semantic and typing principles. Dex uses typed indices for array programming~\citep{maclaurin2019dex,paszke2021dex}, and static shape analysis for TensorFlow programs shows the value of reasoning about tensor shapes at the language level~\citep{lagouvardos2020static}. More broadly, recent work argues that tensor computation benefits from domain-specific language treatment rather than purely ad hoc library interfaces~\citep{bernardy2026domainspecific}.

Tensor programs are especially relevant as an example of a language accompanied by a theorem covering an expressible class. Tensor Programs I introduces a language for neural network computations and proves a Gaussian-process correspondence for networks expressible in that language~\citep{yang2019tensorprograms}. Later tensor programs work develops the program further for feature learning and related infinite-width phenomena~\citep{pmlr-v139-yang21c}. Our paper follows a similar high-level pattern in a different domain. We introduce a language for federated tensor computations and prove factorization theorems for the computations expressible in that language. Our theory is about when federated tensor computation can be carried out through fixed-dimensional shared state.

Federated automatic differentiation studies automatic differentiation for computations involving clients, a server, and communication between them~\citep{rush2024federatedad}. DrJAX develops differentiable MapReduce-style primitives in JAX for distributed and parallel machine learning~\citep{rush2024drjax}. Our differentiable fragment assumes a user-supplied client-local expression for the per-sample gradient. It then proves that the empirical gradient and shared optimizer updates have shared-state factorizations. The theory does not construct this expression by automatic differentiation; the gradient assumption is part of Definition~\ref{def:representably-differentiable-loss}.

Field calculus gives local and aggregate semantics for distributed coordination~\citep{viroli2019aggregate}. Field-based federated learning applies this programming abstraction to coordination under heterogeneous data~\citep{domini2026fbfl}. Our focus is tensor expressions, sample-axis semantics, and shared-state factorizations.

Session types constrain communication protocols and the compatibility of their participants~\citep{honda1998structured,honda2008multiparty,wadler2012sessions}. Our types describe tensor placement and the data axis, from which we prove locality, virtual-global consistency, and factorization for the program classes in Section~\ref{sec:shared-state-factorizations}. These results do not establish session-type properties for an execution backend.

Communication lower bounds for distributed statistical estimation concern the bits required to attain a statistical accuracy target under specified task and protocol assumptions~\citep{braverman2016distributed,zhang2013distributed}. The dimension $D(P)$ in Corollary~\ref{cor:extracted-state-dimension} counts real coordinates in an extracted factorization. It gives a constructive upper bound on state size, without a bit bound or a minimum-state claim. Comparing this construction with those lower bounds requires a statistical task, an accuracy target, and a finite-precision model.

\subsection{Privacy mechanisms and exposure discipline}
\label{app:rw-privacy}

The exposure discipline in this paper is a semantic statement about which tensors become shared. It should not be confused with a cryptographic or differential-privacy guarantee. Secure aggregation, secure multi-party computation, and differential privacy provide distinct privacy mechanisms~\citep{bonawitz2017practical,yao1982protocols,smpc,dwork2006calibrating}. These mechanisms can be used as backends or randomized layers for particular shared-state computations, but they are not implied by the deterministic type system.

At the same time, our typed tensor language gives a precise interface for privacy mechanisms. The shared-state factorization identifies the encoded local messages, the merged shared state, and the decoded output as distinct points where privacy mechanisms may be applied. Appendix~\ref{app:privacy-integration} formalizes this observation through a privacy-lifting result for shared-state factorizations.

\subsection{Summary of our distinct contribution}
\label{app:rw-summary}

The closest existing work provides federated learning algorithms, analytics systems, distributed runtimes, declarative frameworks, tensor languages, differentiable distributed programming, or privacy backends. We define a typed tensor language for shared-state federated computation, prove semantic guarantees for its primitives, and characterize its one-round and iterative program classes through shared-state factorizations. The base class includes analytics, FedSGD, and shared server updates. The transient-state formulation in Appendix~\ref{app:transient-intra-round-state} permits multi-local-step training with a separate factorization guarantee; persistent private client state requires further semantics.

\section{Formal specification of the typed federated tensor language}
\label{app:language-spec}

\subsection{Tensor types, contexts, and expressions}
\label{subsec:tensor-types}

The semantic distinction of Subsection~\ref{subsec:semantic-objects} is reflected at the type level. A shared type records an ordinary tensor shape. A federated type records a record axis together with the common non-record shape. This is required to accommodate federations with different client sample sizes. In what follows, $\mathcal V$ denotes a countably infinite set of variables.

\begin{definition}[Tensor types]
\label{def:tensor-types}
A tensor type is either a shared type $\Sh(\mathbf s)$ or a federated type $\Fed_r(\mathbf d)$.

For a shared type, $\mathbf s = (s_1,\dots,s_k)$ is a tuple for some integer $k \in \mathbb N_0$, where each $s_j$ is a positive integer when $k \geq 1$. The case $k=0$ is identified with the empty tuple and corresponds to a scalar shared type.

For a federated type, there exist an integer $k \geq 1$, an axis $r \in [k]$, and positive integers $d_1,\dots,d_{r-1},d_{r+1},\dots,d_k$ such that $\mathbf d = (d_1,\dots,d_{r-1},d_{r+1},\dots,d_k)$. The type $\Fed_r(\mathbf d)$ denotes a $k$-way federated tensor whose record axis is $r$ and whose common non-record shape is $\mathbf d$. For $k=1$ and $r=1$, the tuple $\mathbf d$ is empty. Local sizes along axis $r$ may be zero; all other extents and all shared extents are positive.
\end{definition}

\begin{definition}[Typing context]
\label{def:typing-context}
A typing context is a finite partial function $\Gamma$ from $\mathcal V$ to tensor types. If $x \in \operatorname{dom}(\Gamma)$, the value $\Gamma(x)$ is called the type of $x$ under $\Gamma$.
\end{definition}

\begin{definition}[Primitive signature]
\label{def:primitive-signature}
A primitive signature is a pair $\Sigma = (\mathcal F,\operatorname{ar})$, where $\mathcal F$ is a set and $\operatorname{ar} : \mathcal F \to \mathbb N_0$ is a function. If $F \in \mathcal F$ and $\operatorname{ar}(F)=m$, then $F$ is called an $m$-ary primitive symbol.
\end{definition}

\begin{definition}[Expressions]
\label{def:expressions}
Let $\Sigma = (\mathcal F,\operatorname{ar})$ be a primitive signature. The set of expressions over $\Sigma$, denoted by $\operatorname{Expr}(\Sigma)$, is the smallest set satisfying the following conditions. If $x \in \mathcal V$, then $x \in \operatorname{Expr}(\Sigma)$. If $F \in \mathcal F$, if $\operatorname{ar}(F)=m$, and if $e_1,\dots,e_m \in \operatorname{Expr}(\Sigma)$, then $F(e_1,\dots,e_m) \in \operatorname{Expr}(\Sigma)$.
\end{definition}

\begin{definition}[Typing judgment]
\label{def:typing-judgment}
Let $\Sigma$ be a primitive signature. A typing judgment over $\Sigma$ is an assertion of the form $\Gamma \vdash e : T$, where $\Gamma$ is a typing context, $e \in \operatorname{Expr}(\Sigma)$, and $T$ is a tensor type.
\end{definition}

Thus, syntax is defined independently of semantic interpretation. A primitive signature specifies only which primitive symbols may occur in expressions and how many inputs each symbol takes. Some typed primitive applications also require structural compatibility conditions that are not stored in the tensor type, such as equality of local record counts or local shapes across federated operands. These conditions are part of the definition of a valid typed primitive application in this specification.

\subsection{Primitive signature}
\label{subsec:primitive-operators}

We now introduce the operator symbols of the base language. The present signature is deterministic. It contains element-wise maps, axis-wise aggregations, axis permutations, and three matrix-product symbols. Later sections enlarge this signature when first-order and second-order learning updates are studied.

\begin{definition}[Aggregation schema]
\label{def:aggregation-schema}
An aggregation schema is a family $\alpha = (\alpha_n)_{n \in \mathbb N_0}$, where $\alpha_n : \RR^n \to \RR$ and $\RR^0 := \{()\}$. The value of $\alpha_0$ is part of the data of $\alpha$.
\end{definition}

Let $\mathcal U$, $\mathcal B$, $\mathcal G$, and $\mathcal A$ be pairwise disjoint collections with the following properties. Each $u \in \mathcal U$ is a map $u : \RR \to \RR$. Each $b \in \mathcal B$ is a map $b : \RR^2 \to \RR$. Each $g \in \mathcal G$ is a map $g : \RR^2 \to \{0,1\} \subseteq \RR$. Each $\alpha \in \mathcal A$ is an aggregation schema in the sense of Definition~\ref{def:aggregation-schema}. For each integer $k \geq 1$, let $\mathfrak S_k$ denote the permutation group of $[k]$, and define $\mathcal P := \bigcup_{k \geq 1} \mathfrak S_k$.

\begin{definition}[Base primitive signature]
\label{def:base-primitive-signature}
The base primitive signature is the pair $\Sigma_0 = (\mathcal F_0,\operatorname{ar})$, where
\[
\mathcal F_0
:=
\mathcal U
\cup
\mathcal B
\cup
\mathcal G
\cup
\{\alpha_j : \alpha \in \mathcal A,\ j \geq 1\}
\cup
\mathcal P
\cup
\{
\mathsf{MatMul}_{\Fed\Sh},
\mathsf{MatMul}_{\Sh\Fed},
\mathsf{MatMul}_{\Fed\Fed}
\},
\]
with all primitive symbols in the displayed union being regarded as formal tagged symbols, and the arity map $\operatorname{ar}$ is given by $\operatorname{ar}(u)=1$, $\operatorname{ar}(b)=2$, $\operatorname{ar}(g)=2$, $\operatorname{ar}(\alpha_j)=1$, and $\operatorname{ar}(\tau)=1$ for every $u \in \mathcal U$, $b \in \mathcal B$, $g \in \mathcal G$, $\alpha \in \mathcal A$, integer $j \geq 1$, and $\tau \in \mathcal P$, together with $\operatorname{ar}(\mathsf{MatMul}_{\Fed\Sh})=\operatorname{ar}(\mathsf{MatMul}_{\Sh\Fed})=\operatorname{ar}(\mathsf{MatMul}_{\Fed\Fed})=2$.
\end{definition}

We use the same notation for a primitive symbol and for its later tensor interpretation. Symbols in $\mathcal U$, $\mathcal B$, and $\mathcal G$ will be interpreted element-wise. Symbols of the form $\alpha_j$ will be interpreted as aggregation along axis $j$. A permutation $\tau \in \mathfrak S_k$ will be interpreted as reordering the axes of a $k$-way tensor according to $\tau$. The three matrix-product symbols distinguish the federated-shared, shared-federated, and federated-federated cases.

The signature records only operator symbols and their arities, without imposing shape compatibility, broadcasting rules, or record-axis constraints. Those conditions enter through the typing rules and the semantic clauses.

\subsection{Ordinary tensor operations induced by the signature}
\label{subsec:ordinary-tensor-ops}

We first interpret the non-matrix primitive symbols on ordinary tensors. These interpretations will later be lifted to shared tensors and federated tensors. Matrix-product symbols are treated separately because their meaning depends on tensor kind and on the position of the record axis.

\begin{definition}[Index sets and axis deletion]
\label{def:index-sets-axis-deletion}
Let $\mathbf s = (s_1,\dots,s_k)$ be a shape tuple with $k \in \mathbb N_0$. Define the index set of $\mathbf s$ by
\[
I(\mathbf s)
=
\begin{cases}
[s_1] \times \ldots \times [s_k], & k \geq 1,\\
\{()\}, & k=0.
\end{cases}
\]
If $k \geq 1$ and $j \in [k]$, define $\mathbf s \setminus j = (s_1,\dots,s_{j-1},s_{j+1},\dots,s_k)$.
\end{definition}

\begin{definition}[Broadcast compatibility]
\label{def:broadcast-compatibility}
Let $\mathbf s = (s_1,\dots,s_p)$ and $\mathbf t = (t_1,\dots,t_q)$ be shape tuples with $p,q \in \mathbb N_0$. Set $\ell := \max\{p,q\}$. Define the left-padded tuples $\operatorname{pad}_{\ell}(\mathbf s) = (\tilde s_1,\dots,\tilde s_{\ell})$ and $\operatorname{pad}_{\ell}(\mathbf t) = (\tilde t_1,\dots,\tilde t_{\ell})$ by
\[
\tilde s_i
=
\begin{cases}
1, & 1 \leq i \leq \ell-p,\\
s_{i-\ell+p}, & \ell-p+1 \leq i \leq \ell,
\end{cases}
\qquad
\tilde t_i
=
\begin{cases}
1, & 1 \leq i \leq \ell-q,\\
t_{i-\ell+q}, & \ell-q+1 \leq i \leq \ell.
\end{cases}
\]
The tuples $\mathbf s$ and $\mathbf t$ are broadcast-compatible if, for every $i \in [\ell]$, at least one of the following holds: $\tilde s_i = \tilde t_i$, $\tilde s_i = 1$, or $\tilde t_i = 1$. If $\mathbf s$ and $\mathbf t$ are broadcast-compatible, their broadcast shape is $\mathbf s \vee \mathbf t = (u_1,\dots,u_{\ell})$. For $i \in [\ell]$, set
\[
u_i :=
\begin{cases}
\tilde t_i, & \tilde s_i=1,\\
\tilde s_i, & \tilde s_i\ne1.
\end{cases}
\]
Thus, broadcasting an extent of zero against one retains zero.
\end{definition}

\begin{definition}[Broadcast map]
\label{def:broadcast-map}
Let $\mathbf s$ and $\mathbf u$ be shape tuples such that $\mathbf s$ is broadcast-compatible with $\mathbf u$ and such that $\mathbf s \vee \mathbf u = \mathbf u$. Let $\mathbf s = (s_1,\dots,s_p)$ and $\mathbf u = (u_1,\dots,u_{\ell})$, with $p,\ell \in \mathbb N_0$ and $p \leq \ell$. Write $\operatorname{pad}_{\ell}(\mathbf s) = (\tilde s_1,\dots,\tilde s_{\ell})$. Define $\pi_{\mathbf s}^{\mathbf u} : I(\mathbf u) \to I(\mathbf s)$ as follows. If $p=0$, then $\pi_{\mathbf s}^{\mathbf u}(i)=()$ for all $i \in I(\mathbf u)$. If $p \geq 1$ and $i = (i_1,\dots,i_{\ell}) \in I(\mathbf u)$, then $\pi_{\mathbf s}^{\mathbf u}(i) = (j_1,\dots,j_p)$, where for each $q \in [p]$,
\[
j_q
=
\begin{cases}
1, & s_q=1,\\
i_{\ell-p+q}, & s_q\ne1.
\end{cases}
\]
Define the broadcast map $\operatorname{Br}_{\mathbf s}^{\mathbf u} : \RR^{\mathbf s} \to \RR^{\mathbf u}$ as follows. If $T \in \RR^{\mathbf s}$, then $\operatorname{Br}_{\mathbf s}^{\mathbf u}(T) \in \RR^{\mathbf u}$ is defined by $\operatorname{Br}_{\mathbf s}^{\mathbf u}(T)[i] = T\left[\pi_{\mathbf s}^{\mathbf u}(i)\right]$ for every $i \in I(\mathbf u)$.
\end{definition}

\begin{definition}[Ordinary tensor operations induced by the signature]
\label{def:tensor_ops}
Let $\Sigma_0$ be the base primitive signature of Subsection~\ref{subsec:primitive-operators}.

If $u \in \mathcal U$ and $T \in \RR^{\mathbf s}$, define $u(T) \in \RR^{\mathbf s}$ by $u(T)[i] = u(T[i])$ for every $i \in I(\mathbf s)$.

If $b \in \mathcal B$, $S \in \RR^{\mathbf s}$, and $T \in \RR^{\mathbf t}$, assume that $\mathbf s$ and $\mathbf t$ are broadcast-compatible and set $\mathbf u := \mathbf s \vee \mathbf t$. Define $b(S,T) \in \RR^{\mathbf u}$ by $b(S,T)[i] = b\left(\operatorname{Br}_{\mathbf s}^{\mathbf u}(S)[i],\operatorname{Br}_{\mathbf t}^{\mathbf u}(T)[i]\right)$ for every $i \in I(\mathbf u)$.

If $g \in \mathcal G$, $S \in \RR^{\mathbf s}$, and $T \in \RR^{\mathbf t}$, assume that $\mathbf s$ and $\mathbf t$ are broadcast-compatible and set $\mathbf u := \mathbf s \vee \mathbf t$. Define $g(S,T) \in \RR^{\mathbf u}$ by $g(S,T)[i] = g\left(\operatorname{Br}_{\mathbf s}^{\mathbf u}(S)[i],\operatorname{Br}_{\mathbf t}^{\mathbf u}(T)[i]\right)$ for every $i \in I(\mathbf u)$.

If $\alpha \in \mathcal A$, $k \geq 1$, $T \in \RR^{s_1 \times \ldots \times s_k}$, and $j \in [k]$, define $\alpha_j(T) \in \RR^{\mathbf s \setminus j}$, where $\mathbf s := (s_1,\dots,s_k)$, by the rule
\begin{multline*}
\alpha_j(T)[i_1,\dots,i_{j-1},i_{j+1},\dots,i_k] = \\
\alpha_{s_j}\left(T[i_1,\dots,i_{j-1},1,i_{j+1},\dots,i_k],\dots,T[i_1,\dots,i_{j-1},s_j,i_{j+1},\dots,i_k]\right)
\end{multline*}
for every $(i_1,\dots,i_{j-1},i_{j+1},\dots,i_k) \in I(\mathbf s \setminus j)$. If $k=1$ and $j=1$, then $\alpha_1(T) \in \RR$ is given by $\alpha_1(T) = \alpha_{s_1}(T[1],\dots,T[s_1])$.

If $\tau \in \mathfrak S_k$ and $T \in \RR^{s_1 \times \cdots \times s_k}$, define $\tau \cdot (s_1,\dots,s_k) = (s_{\tau^{-1}(1)},\dots,s_{\tau^{-1}(k)})$, and define $\tau(T) \in \RR^{\tau \cdot (s_1,\dots,s_k)}$ by $\tau(T)[i_1,\dots,i_k] = T[i_{\tau(1)},\dots,i_{\tau(k)}]$ for every $(i_1,\dots,i_k) \in I(\tau \cdot (s_1,\dots,s_k))$.
\end{definition}

Definition~\ref{def:tensor_ops} gives the ordinary tensor meaning of unary maps, binary element-wise maps, aggregation symbols, and axis permutations. These ordinary operations will be combined with the two tensor kinds of Subsection~\ref{subsec:semantic-objects} in the next subsection.

\subsection{Semantics of non-matrix primitives on shared and federated tensors}
\label{subsec:primitive-semantics}

We interpret the non-matrix symbols of $\Sigma_0$ on shared tensors and federated tensors. Matrix-product symbols are treated separately because their output kind depends on how the record axis is used.

\begin{definition}[Ordinary shape]
\label{def:ordinary-shape}
Let $T$ be an ordinary tensor or a scalar. There exists a unique shape tuple $\mathbf s$ such that $T \in \RR^{\mathbf s}$. Define the ordinary shape of $T$ as $\operatorname{shape}(T) := \mathbf s$.
\end{definition}

\begin{definition}[Local broadcast of a shared tensor against a federated tensor]
\label{def:local-broadcast}
Let $X=\{X^{(c)}\}_{c\in\clients}$ be a federated tensor. Let $S$ be a shared tensor. Assume that, for every $c \in \clients$, the tuples $\operatorname{shape}(S)$ and $\operatorname{shape}(X^{(c)})$ are broadcast-compatible and satisfy $\operatorname{shape}(S) \vee \operatorname{shape}(X^{(c)}) = \operatorname{shape}(X^{(c)})$. For each $c \in \clients$, define the local broadcast of $S$ against $X$ as $\operatorname{Br}_c(S;X) := \operatorname{Br}_{\operatorname{shape}(S)}^{\operatorname{shape}(X^{(c)})}(S) \in \RR^{\operatorname{shape}(X^{(c)})}$.
\end{definition}

\begin{definition}[Unary semantics]
\label{def:unary-semantics}
Let $u \in \mathcal U$. If $S$ is a shared tensor, define $u(S)$ to be the shared tensor obtained by applying the ordinary tensor operation $u$ of Definition~\ref{def:tensor_ops} to $S$. If $X=\{X^{(c)}\}_{c\in\clients}$ is a federated tensor with record axis $r$, define the unary semantics of $X$ as $u(X) := \{u(X^{(c)})\}_{c\in\clients}$.
\end{definition}

\begin{definition}[Binary element-wise semantics]
\label{def:binary-elementwise-semantics}
Let $\beta \in \mathcal B \cup \mathcal G$.

If $S$ and $T$ are shared tensors and $\operatorname{shape}(S)$ and $\operatorname{shape}(T)$ are broadcast-compatible, define the binary element-wise semantics $\beta(S,T)$ of $S$ and $T$ as the shared tensor obtained by applying the ordinary tensor operation $\beta$ of Definition~\ref{def:tensor_ops} to $S$ and $T$.

If $X=\{X^{(c)}\}_{c\in\clients}$ and $Z=\{Z^{(c)}\}_{c\in\clients}$ are federated tensors with the same record axis $r$ and satisfying $\operatorname{shape}(X^{(c)})=\operatorname{shape}(Z^{(c)})$ for every $c \in \clients$, define $\beta(X,Z) := \{\beta(X^{(c)},Z^{(c)})\}_{c\in\clients}$.

If $X=\{X^{(c)}\}_{c\in\clients}$ is a federated tensor with record axis $r$, if $S$ is a shared tensor, and if $\operatorname{Br}_c(S;X)$ is defined for every $c \in \clients$, define $\beta(X,S) := \{\beta(X^{(c)},\operatorname{Br}_c(S;X))\}_{c\in\clients}$.

If $S$ is a shared tensor, if $X=\{X^{(c)}\}_{c\in\clients}$ is a federated tensor with record axis $r$, and if $\operatorname{Br}_c(S;X)$ is defined for every $c \in \clients$, define $\beta(S,X) := \{\beta(\operatorname{Br}_c(S;X),X^{(c)})\}_{c\in\clients}$.
\end{definition}

\begin{definition}[Aggregation semantics]
\label{def:aggregation-semantics}
Let $\alpha \in \mathcal A$.

If $S$ is a shared tensor of ordinary order $k \geq 1$ and if $j \in [k]$, define $\alpha_j(S)$ to be the shared tensor obtained by applying the ordinary tensor operation $\alpha_j$ of Definition~\ref{def:tensor_ops} to $S$.

If $X=\{X^{(c)}\}_{c\in\clients}$ is a federated tensor of ordinary order $k \geq 1$ with record axis $r$ and if $j=r$, define the aggregation semantics of $X$ along axis $r$ as $\alpha_r(X) := \alpha_r(\vglob{X})$.

If $X=\{X^{(c)}\}_{c\in\clients}$ is a federated tensor of ordinary order $k \geq 1$ with record axis $r$ and if $j \in [k] \setminus \{r\}$, define $\alpha_j(X) := \{\alpha_j(X^{(c)})\}_{c\in\clients}$. Then $\alpha_j(X)$ is a federated tensor with record axis
\[
r'=
\begin{cases}
r, & j>r,\\
r-1, & j<r.
\end{cases}
\]
\end{definition}

\begin{definition}[Permutation semantics]
\label{def:permutation-semantics}
Let $k \geq 1$ and let $\tau \in \mathfrak S_k$.

If $S$ is a shared tensor of ordinary order $k$, define the permutation semantics $\tau(S)$ of $S$ to be the shared tensor obtained by applying the ordinary tensor operation $\tau$ of Definition~\ref{def:tensor_ops} to $S$.

If $X=\{X^{(c)}\}_{c\in\clients}$ is a federated tensor of ordinary order $k$ with record axis $r$, define $\tau(X) := \{\tau(X^{(c)})\}_{c\in\clients}$.
\end{definition}

It follows from Definitions~\ref{def:unary-semantics}, \ref{def:binary-elementwise-semantics}, \ref{def:aggregation-semantics}, and \ref{def:permutation-semantics} that $u(X)$ is a federated tensor with record axis $r$, each of $\beta(X,Z)$, $\beta(X,S)$, and $\beta(S,X)$ is a federated tensor with record axis $r$, $\alpha_r(X)$ is a shared tensor, and $\tau(X)$ is a federated tensor with record axis $\tau(r)$. These semantic definitions specify how non-matrix primitives act on shared tensors and federated tensors. They also identify the only non-matrix operation that can produce shared state from federated inputs, namely aggregation along the record axis.

\subsection{Semantics of matrix-product primitives}
\label{subsec:matrix-semantics}

We now interpret the three matrix-product symbols in $\Sigma_0$. Each symbol is partial. It is defined for inputs whose matrix shapes and record-axis positions satisfy the stated conditions. These are the cases needed to preserve the federated kind or to eliminate the record axis into shared state.

A shared matrix is a shared tensor of ordinary order $2$. A federated matrix is a federated tensor of ordinary order $2$.

\begin{definition}[Federated-shared matrix product]
\label{def:fed-shared-matmul-semantics}
Let $X=\{X^{(c)}\}_{c\in\clients}$ be a federated matrix with record axis $1$. Assume that there exists an integer $p \geq 1$ and a family of local record counts $\{n_c\}_{c\in\clients} \subseteq \mathbb N_0$ such that $X^{(c)} \in \RR^{n_c \times p}$ for every $c \in \clients$. Let $S \in \RR^{p \times q}$ be a shared matrix for some integer $q \geq 1$. Define the federated-shared matrix product as $\mathsf{MatMul}_{\Fed\Sh}(X,S) := \{X^{(c)}S\}_{c\in\clients}$.
\end{definition}

\begin{definition}[Shared-federated matrix product]
\label{def:shared-fed-matmul-semantics}
Let $X=\{X^{(c)}\}_{c\in\clients}$ be a federated matrix with record axis $2$. Assume that there exists an integer $p \geq 1$ and a family of local record counts $\{n_c\}_{c\in\clients} \subseteq \mathbb N_0$ such that $X^{(c)} \in \RR^{p \times n_c}$ for every $c \in \clients$. Let $S \in \RR^{q \times p}$ be a shared matrix for some integer $q \geq 1$. Define the shared-federated matrix product as $\mathsf{MatMul}_{\Sh\Fed}(S,X) := \{SX^{(c)}\}_{c\in\clients}$.
\end{definition}

\begin{definition}[Federated-federated matrix product]
\label{def:fed-fed-matmul-semantics}
Let $X=\{X^{(c)}\}_{c\in\clients}$ and $Z=\{Z^{(c)}\}_{c\in\clients}$ be federated matrices. Assume that $X$ has record axis $2$, that $Z$ has record axis $1$, and that there exist integers $a,b \geq 1$ and a family of local record counts $\{n_c\}_{c\in\clients} \subseteq \mathbb N_0$ such that $X^{(c)} \in \RR^{a \times n_c}$ and $Z^{(c)} \in \RR^{n_c \times b}$ for every $c \in \clients$. Define the federated-federated matrix product as $\mathsf{MatMul}_{\Fed\Fed}(X,Z) := \sum_{c\in\clients} X^{(c)}Z^{(c)}$.
\end{definition}

It follows from Definitions~\ref{def:fed-shared-matmul-semantics}, \ref{def:shared-fed-matmul-semantics}, and \ref{def:fed-fed-matmul-semantics} that $\mathsf{MatMul}_{\Fed\Sh}(X,S)$ is a federated matrix with record axis $1$ and local shapes $X^{(c)}S \in \RR^{n_c \times q}$ for every $c \in \clients$, $\mathsf{MatMul}_{\Sh\Fed}(S,X)$ is a federated matrix with record axis $2$ and local shapes $SX^{(c)} \in \RR^{q \times n_c}$ for every $c \in \clients$, and $\mathsf{MatMul}_{\Fed\Fed}(X,Z)$ is a shared matrix in $\RR^{a \times b}$.

The federated-shared and shared-federated matrix products preserve the federated structure. The federated-federated matrix product eliminates the record axis and produces shared state. Other axis placements can be reduced to these cases by composing with permutation primitives. Shared-only matrix multiplication, which is not part of the base signature, will be added when later sections expand the language for shared-only post-processing and second-order updates.

\subsection{Type-level shape operations}
\label{subsec:type-level-ops}

We introduce the type-level operations needed to state the typing rules. The main issue is that a federated type records a record axis but does not record local record counts. For that reason, type-level broadcasting against a federated type must be defined without reference to any client sample size.

\begin{definition}[Symbolic shape of a federated type]
\label{def:symbolic-shape}
Let $T = \Fed_r(\mathbf d)$, $\mathbf d = (d_1,\dots,d_{k-1})$, where $k \geq 1$ and $r \in [k]$. In this definition, the $k-1$ entries of $\mathbf d$ are indexed consecutively, preserving their order after axis $r$ is omitted. Define the symbolic shape of $T$ as the tuple $\operatorname{sym}(T) = (\sigma_1,\dots,\sigma_k)$ with entries in $\mathbb N \cup \{\star\}$ given by
\[
\sigma_j
=
\begin{cases}
d_j, & 1 \leq j < r,\\
\star, & j=r,\\
d_{j-1}, & r < j \leq k.
\end{cases}
\]
\end{definition}

The tuple $\operatorname{sym}(T)$ of Definition~\ref{def:symbolic-shape} has exactly one coordinate equal to $\star$. That coordinate records the record axis.

\begin{definition}[Erasure of the record marker]
\label{def:erase-record-marker}
Let $\boldsymbol{\sigma} = (\sigma_1,\dots,\sigma_k)$ be a tuple with entries in $\mathbb N \cup \{\star\}$ and with exactly one coordinate equal to $\star$. Define the erasure $\operatorname{erase}_{\star}(\boldsymbol{\sigma})$ of the record marker in $\boldsymbol{\sigma}$ as the tuple of positive integers obtained by deleting the unique coordinate equal to $\star$ and preserving the order of the remaining coordinates.
\end{definition}

\begin{definition}[Record-agnostic broadcast compatibility]
\label{def:record-agnostic-broadcast-compatibility}
Let $\mathbf s = (s_1,\dots,s_p)$ be a shared shape tuple with $p \in \mathbb N_0$. Let $T = \Fed_r(\mathbf d)$ be a federated type whose symbolic shape is $\operatorname{sym}(T) = (\sigma_1,\dots,\sigma_k)$. If $p>k$, then $\mathbf s$ is not record-agnostically broadcast-compatible with $T$. Assume henceforth that $p \leq k$. Define the left-padded tuple $\operatorname{pad}_{k}(\mathbf s) = (\tilde s_1,\dots,\tilde s_k)$ by padding $\mathbf s$ on the left with entries equal to $1$.

We say that $\mathbf s$ is record-agnostically broadcast-compatible with $T$ if, for every $i \in [k]$, the following conditions hold. If $\sigma_i = \star$, then $\tilde s_i = 1$. If $\sigma_i \in \mathbb N$, then $\tilde s_i \in \{1,\sigma_i\}$.
\end{definition}

The previous condition depends only on the type $\Fed_r(\mathbf d)$, without a dependence on any local record count. It is the static condition that guarantees a shared tensor of shape $\mathbf s$ can be broadcast against every local tensor having federated type $T$, without changing the local tensor shape.

\begin{definition}[Deletion of a non-record axis from a federated type]
\label{def:delete-nonrecord-axis}
Let $T = \Fed_r(\mathbf d)$, $\operatorname{sym}(T) = (\sigma_1,\dots,\sigma_k)$, with $k \geq 1$. Let $j \in [k] \setminus \{r\}$. Define the deletion of axis $j$ from $T$ as follows. Let $\boldsymbol{\sigma}' = (\sigma_1,\dots,\sigma_{j-1},\sigma_{j+1},\dots,\sigma_k)$. Let $r'$ be the unique index in $[k-1]$ such that the $r'$-th coordinate of $\boldsymbol{\sigma}'$ is equal to $\star$. Let $\mathbf d' := \operatorname{erase}_{\star}(\boldsymbol{\sigma}')$ as in Definition~\ref{def:erase-record-marker}. Define $\operatorname{del}_j(T) := \Fed_{r'}(\mathbf d')$.
\end{definition}

\begin{definition}[Permutation of a federated type]
\label{def:permute-federated-type}
Let $T = \Fed_r(\mathbf d)$, $\operatorname{sym}(T) = (\sigma_1,\dots,\sigma_k)$, with $k \geq 1$, and let $\tau \in \mathfrak S_k$. Define the permutation of $T$ by $\tau$ as follows. Define $\tau \cdot \operatorname{sym}(T) = (\sigma_{\tau^{-1}(1)},\dots,\sigma_{\tau^{-1}(k)})$. Let $\mathbf d' := \operatorname{erase}_{\star}\left(\tau \cdot \operatorname{sym}(T)\right)$ as in Definition~\ref{def:erase-record-marker}. Define $\tau(T) := \Fed_{\tau(r)}(\mathbf d')$.
\end{definition}

\subsection{Typing rules for non-matrix primitives}
\label{subsec:typing-rules-nonmatrix}

In this subsection, all expressions are expressions over the base primitive signature $\Sigma_0$. The typing rules below assign output types to primitive applications whenever the corresponding semantic compatibility conditions are satisfied. For every integer $k \geq 1$, every permutation $\tau \in \mathfrak S_k$, and every shape tuple $\mathbf s = (s_1,\dots,s_k)$, we write $\tau \cdot \mathbf s := (s_{\tau^{-1}(1)},\dots,s_{\tau^{-1}(k)})$.

\begin{definition}[Typing of unary primitives]
\label{def:typing-unary}
Let $u \in \mathcal U$.

If $\Gamma \vdash e : \Sh(\mathbf s)$, then define the typing rule for $u$ on $e$ by $\Gamma \vdash u(e) : \Sh(\mathbf s)$.

If $\Gamma \vdash e : \Fed_r(\mathbf d)$, then define the typing rule for $u$ on $e$ by $\Gamma \vdash u(e) : \Fed_r(\mathbf d)$.
\end{definition}

\begin{definition}[Typing of binary element-wise primitives]
\label{def:typing-binary-elementwise}
Let $\beta \in \mathcal B \cup \mathcal G$.

If $\Gamma \vdash e : \Sh(\mathbf s)$ and $\Gamma \vdash f : \Sh(\mathbf t)$, and if $\mathbf s$ and $\mathbf t$ are broadcast-compatible, then define the typing rule for $\beta$ on $(e,f)$ by $\Gamma \vdash \beta(e,f) : \Sh(\mathbf s \vee \mathbf t)$.

If $\Gamma \vdash e : \Fed_r(\mathbf d)$ and $\Gamma \vdash f : \Fed_r(\mathbf d)$, and if the corresponding federated primitive application satisfies the local-shape compatibility condition of Definition~\ref{def:binary-elementwise-semantics}, then define the typing rule for $\beta$ on $(e,f)$ by $\Gamma \vdash \beta(e,f) : \Fed_r(\mathbf d)$.

If $\Gamma \vdash e : \Fed_r(\mathbf d)$ and $\Gamma \vdash f : \Sh(\mathbf s)$, and if $\mathbf s$ is record-agnostically broadcast-compatible with $\Fed_r(\mathbf d)$ according to Definition~\ref{def:record-agnostic-broadcast-compatibility}, then define the typing rule for $\beta$ on $(e,f)$ by $\Gamma \vdash \beta(e,f) : \Fed_r(\mathbf d)$.

If $\Gamma \vdash e : \Sh(\mathbf s)$ and $\Gamma \vdash f : \Fed_r(\mathbf d)$, and if $\mathbf s$ is record-agnostically broadcast-compatible with $\Fed_r(\mathbf d)$ according to Definition~\ref{def:record-agnostic-broadcast-compatibility}, then define the typing rule for $\beta$ on $(e,f)$ by $\Gamma \vdash \beta(e,f) : \Fed_r(\mathbf d)$.
\end{definition}

\begin{definition}[Typing of aggregation primitives]
\label{def:typing-aggregation}
Let $\alpha \in \mathcal A$. Let $\mathbf s = (s_1,\dots,s_k)$ with $k \geq 1$. If $\Gamma \vdash e : \Sh(\mathbf s)$ and $j \in [k]$, then define the typing rule for $\alpha_j$ on $e$ by $\Gamma \vdash \alpha_j(e) : \Sh(\mathbf s \setminus j)$.

Let $T = \Fed_r(\mathbf d)$, where $r \in [k]$ and $k = |\mathbf d|+1$. If $\Gamma \vdash e : T$, then define the typing rule for record-axis aggregation on $e$ by $\Gamma \vdash \alpha_r(e) : \Sh(\mathbf d)$.

If $\Gamma \vdash e : T$ and $j \in [k] \setminus \{r\}$, then define the typing rule for non-record-axis aggregation on $e$ by $\Gamma \vdash \alpha_j(e) : \operatorname{del}_j(T)$, where $\operatorname{del}_j(T)$ is the federated type of Definition~\ref{def:delete-nonrecord-axis}.
\end{definition}

\begin{definition}[Typing of permutation primitives]
\label{def:typing-permutation}
Let $k \geq 1$ and let $\tau \in \mathfrak S_k$.

If $\Gamma \vdash e : \Sh(\mathbf s)$, where $\mathbf s = (s_1,\dots,s_k)$, then define the typing rule for $\tau$ on $e$ by $\Gamma \vdash \tau(e) : \Sh(\tau \cdot \mathbf s)$.

If $\Gamma \vdash e : \Fed_r(\mathbf d)$ and $k = |\mathbf d|+1$, then define the typing rule for $\tau$ on $e$ by $\Gamma \vdash \tau(e) : \tau(\Fed_r(\mathbf d))$, where $\tau(\Fed_r(\mathbf d))$ is the federated type of Definition~\ref{def:permute-federated-type}.
\end{definition}

The rules of Definitions~\ref{def:typing-unary}--\ref{def:typing-permutation} provide the two ways in which a primitive may act on federated data. It may preserve federated structure, as in unary maps, element-wise operations, non-record-axis aggregation, and permutations. It may also eliminate the record axis and produce shared state, as in record-axis aggregation. The remaining shared-output primitive in the base language is the federated-federated matrix product, which is typed separately.

\subsection{Typing rules for matrix-product primitives}
\label{subsec:typing-rules-matrix}

We now state the typing rules for the three matrix-product primitives in the base language. These are the only matrix products in $\Sigma_0$. They are chosen so that matrix multiplication either preserves the federated structure or eliminates the record axis into shared state.

\begin{definition}[Typing of the federated-shared matrix product]
\label{def:typing-fed-shared-matmul}
Let $p,q \geq 1$. If $\Gamma \vdash e : \Fed_1((p))$ and $\Gamma \vdash f : \Sh((p,q))$, then define the typing rule for the federated-shared matrix product by $\Gamma \vdash \mathsf{MatMul}_{\Fed\Sh}(e,f) : \Fed_1((q))$.
\end{definition}

\begin{definition}[Typing of the shared-federated matrix product]
\label{def:typing-shared-fed-matmul}
Let $p,q \geq 1$. If $\Gamma \vdash e : \Sh((q,p))$ and $\Gamma \vdash f : \Fed_2((p))$, then define the typing rule for the shared-federated matrix product by $\Gamma \vdash \mathsf{MatMul}_{\Sh\Fed}(e,f) : \Fed_2((q))$.
\end{definition}

\begin{definition}[Typing of the federated-federated matrix product]
\label{def:typing-fed-fed-matmul}
Let $a,b \geq 1$. If $\Gamma \vdash e : \Fed_2((a))$ and $\Gamma \vdash f : \Fed_1((b))$, and if the corresponding primitive application satisfies the client-wise local-count compatibility condition of Definition~\ref{def:fed-fed-matmul-semantics}, then define the typing rule for the federated-federated matrix product by $\Gamma \vdash \mathsf{MatMul}_{\Fed\Fed}(e,f) : \Sh((a,b))$.
\end{definition}

The typing rules of Definitions~\ref{def:typing-fed-shared-matmul} and~\ref{def:typing-shared-fed-matmul} preserve the federated kind. The typing rule of Definition~\ref{def:typing-fed-fed-matmul} produces a shared tensor. Other placements of the record axis can be reduced to these cases by combining matrix-product primitives with permutation primitives. Because federated types do not record local record counts, the federated-federated matrix product rule carries a semantic side condition. It applies only when the locally contracted dimensions agree on every client.

\section{Proofs of foundational semantic propositions}
\label{app:semantic-proofs}

\begin{proof}[Proof of Proposition~\ref{prop:client-locality}]
Fix a primitive symbol $F$ of $\Sigma_0$. Consider a typed primitive application of $F$ whose federated inputs are $X_1,\dots,X_p$ and whose shared inputs are $S_1,\dots,S_q$, and let $Y$ denote its output. Assume that $Y$ is a federated tensor. We show that, for each client $c \in \clients$, the local output $Y^{(c)}$ depends only on the local federated inputs at $c$ and on the shared inputs.

We proceed by a case analysis over the federated-output primitives of the base language.

Assume first that $F=u$ for some $u \in \mathcal U$. Then there is one federated input and no shared input. By Definition~\ref{def:unary-semantics}, $Y^{(c)} = u(X_1^{(c)})$ for every $c \in \clients$. Hence the required map is $\Phi_{F,c}(A) := u(A)$.

Assume next that $F=\beta$ for some $\beta \in \mathcal B \cup \mathcal G$ and that both inputs are federated. Then there are two federated inputs and no shared input. By Definition~\ref{def:binary-elementwise-semantics}, $Y^{(c)} = \beta(X_1^{(c)},X_2^{(c)})$ for every $c \in \clients$. Hence the required map is $\Phi_{F,c}(A,B) := \beta(A,B)$.

Assume next that $F=\beta$ for some $\beta \in \mathcal B \cup \mathcal G$, that there is one federated input $X_1$, and that there is one shared input $S_1$ used in the second operand position. By Definition~\ref{def:binary-elementwise-semantics}, $Y^{(c)} = \beta(X_1^{(c)},\operatorname{Br}_c(S_1;X_1))$ for every $c \in \clients$. For fixed $c$, define $\Phi_{F,c}(A,S) := \beta(A,\operatorname{Br}_{\operatorname{shape}(S)}^{\operatorname{shape}(A)}(S))$ for tensors $A$ and $S$ of the same shapes as $X_1^{(c)}$ and $S_1$, respectively (see Definitions~\ref{def:ordinary-shape} and~\ref{def:broadcast-map}). Substituting $A=X_1^{(c)}$ and $S=S_1$ gives $Y^{(c)} = \Phi_{F,c}(X_1^{(c)},S_1)$.

Assume next that $F=\beta$ for some $\beta \in \mathcal B \cup \mathcal G$, that there is one shared input $S_1$ used in the first operand position, and that there is one federated input $X_1$. By Definition~\ref{def:binary-elementwise-semantics}, $Y^{(c)} = \beta(\operatorname{Br}_c(S_1;X_1),X_1^{(c)})$ for every $c \in \clients$. For fixed $c$, define $\Phi_{F,c}(A,S) := \beta(\operatorname{Br}_{\operatorname{shape}(S)}^{\operatorname{shape}(A)}(S),A)$ for tensors $A$ and $S$ of the same shapes as $X_1^{(c)}$ and $S_1$, respectively. Substituting $A=X_1^{(c)}$ and $S=S_1$ gives $Y^{(c)} = \Phi_{F,c}(X_1^{(c)},S_1)$.

Assume next that $F=\alpha_j$ for some $\alpha \in \mathcal A$ and some axis $j$. Since the output is federated, Definition~\ref{def:typing-aggregation} implies that the input is a federated tensor $X_1$ and that $j$ is not its record axis. By Definition~\ref{def:aggregation-semantics}, $Y^{(c)} = \alpha_j(X_1^{(c)})$ for every $c \in \clients$. Hence the required map is $\Phi_{F,c}(A) := \alpha_j(A)$.

Assume next that $F=\tau$ for some permutation $\tau \in \mathfrak S_k$. Since the output is federated, the input is a federated tensor $X_1$ of ordinary order $k$ by Definition~\ref{def:typing-permutation}. By Definition~\ref{def:permutation-semantics}, $Y^{(c)} = \tau(X_1^{(c)})$ for every $c \in \clients$. Hence the required map is $\Phi_{F,c}(A) := \tau(A)$.

Assume next that $F=\mathsf{MatMul}_{\Fed\Sh}$. Then there is one federated input $X_1$ and one shared input $S_1$. By Definition~\ref{def:fed-shared-matmul-semantics}, $Y^{(c)} = X_1^{(c)}S_1$ for every $c \in \clients$. Hence the required map is $\Phi_{F,c}(A,S) := AS$.

Assume finally that $F=\mathsf{MatMul}_{\Sh\Fed}$. Then there is one shared input $S_1$ and one federated input $X_1$. By Definition~\ref{def:shared-fed-matmul-semantics}, $Y^{(c)} = S_1X_1^{(c)}$ for every $c \in \clients$. Hence the required map is $\Phi_{F,c}(A,S) := SA$.

These cases exhaust the federated-output primitives of the base language by Definitions~\ref{def:unary-semantics}, \ref{def:binary-elementwise-semantics}, \ref{def:aggregation-semantics}, \ref{def:permutation-semantics}, \ref{def:fed-shared-matmul-semantics}, and~\ref{def:shared-fed-matmul-semantics}. In each case, the constructed map is determined by the primitive symbol, the operand positions, and the local input shapes. It does not depend on the identity of $c$ or on any tensor stored at another client. Therefore, for every client $c \in \clients$, there exists a deterministic map $\Phi_{F,c}$ such that $Y^{(c)} = \Phi_{F,c}(X_1^{(c)},\dots,X_p^{(c)},S_1,\dots,S_q)$. This proves the proposition.
\end{proof}

\begin{proof}[Proof of Proposition~\ref{prop:exposure-discipline}]
Fix a primitive symbol $F$ of $\Sigma_0$. Consider a primitive application of $F$ whose federated inputs are $X_1,\dots,X_p$ and whose shared inputs are $S_1,\dots,S_q$, and let $Y$ denote its output.

The convention for global availability in Subsection~\ref{subsec:semantic-objects} gives the first claim, including when the federation has one client.

We now prove the second claim. Assume that $p \geq 1$ and that $Y$ is shared.

Assume first that $F=u$ for some $u \in \mathcal U$. By Definition~\ref{def:unary-semantics}, a unary primitive applied to a federated tensor produces a federated tensor. Since $p \geq 1$, the unique input is federated. Hence $Y$ cannot be shared. This case is impossible.

Assume next that $F=\beta$ for some $\beta \in \mathcal B \cup \mathcal G$. By Definition~\ref{def:binary-elementwise-semantics}, if at least one input is federated, then the output is federated. Since $p \geq 1$, at least one input is federated. Hence $Y$ cannot be shared. This case is impossible.

Assume next that $F=\tau$ for some permutation $\tau \in \mathfrak S_k$. By Definition~\ref{def:permutation-semantics}, a permutation primitive preserves tensor kind. Since $p \geq 1$, the unique input is federated. Hence $Y$ cannot be shared. This case is impossible.

Assume next that $F=\alpha_j$ for some $\alpha \in \mathcal A$ and some axis $j$. Since aggregation primitives are unary, the unique input is federated because $p \geq 1$. By Definition~\ref{def:aggregation-semantics}, aggregation along the record axis produces a shared tensor, while aggregation along any non-record axis produces a federated tensor. Therefore, if $Y$ is shared, there exist $\alpha \in \mathcal A$, an index $i \in [p]$, and the record axis $r$ of $X_i$ such that $Y = \alpha_r(X_i)$.

Assume next that $F=\mathsf{MatMul}_{\Fed\Sh}$. By Definition~\ref{def:fed-shared-matmul-semantics}, the output is federated. Hence $Y$ cannot be shared. This case is impossible.

Assume next that $F=\mathsf{MatMul}_{\Sh\Fed}$. By Definition~\ref{def:shared-fed-matmul-semantics}, the output is federated. Hence $Y$ cannot be shared. This case is impossible.

Assume finally that $F=\mathsf{MatMul}_{\Fed\Fed}$. By Definition~\ref{def:fed-fed-matmul-semantics}, the output is shared, and both operands are federated matrices. Therefore there exist indices $i,j \in [p]$ such that $Y = \mathsf{MatMul}_{\Fed\Fed}(X_i,X_j)$. The two operands have record axes $2$ and $1$, respectively, and each federated input has one designated record axis, so $i\ne j$.

These cases exhaust the primitive symbols of $\Sigma_0$ by Definition~\ref{def:base-primitive-signature}. Since the aggregation symbols and the matrix-product symbols are distinct primitive symbols, the two stated alternatives are mutually exclusive. Therefore, under the assumptions $p \geq 1$ and $Y$ shared, exactly one of the two cases holds. This proves the proposition.
\end{proof}

\begin{proof}[Proof of Proposition~\ref{prop:virtual-global-consistency}]
Fix a primitive symbol $F$ of $\Sigma_0$ with arity $m$. Consider a typed primitive application of $F$ to inputs $A_1,\dots,A_m$, where each $A_t$ is either a federated tensor or a shared tensor, and let $Y$ denote its output. For each $t \in [m]$, define $\widetilde A_t$ as in the statement of the proposition.

For each federated tensor $X$, write $\clients = \{c_1,\dots,c_M\}$ with $c_1 < \ldots < c_M$. If $X$ has record axis $r$ and local record counts $n_{c_1}(X),\dots,n_{c_M}(X)$, define the client offset $o_X(c_j) := \sum_{\ell=1}^{j-1} n_{c_\ell}(X)$ for $j \in [M]$. If $i=(i_1,\dots,i_k)$ is a local index of $X^{(c_j)}$, define the corresponding global index $\mu_X(c_j,i)=(h_1,\dots,h_k)$ by $h_t=i_t$ for $t \neq r$ and $h_r=o_X(c_j)+i_r$. By Definition~\ref{def:virtual-global-tensor}, $X^{(c_j)}[i] = \vglob{X}[\mu_X(c_j,i)]$ for every local index $i$.

We first treat the case in which $Y$ is federated.

Assume first that $F=u$ for some $u \in \mathcal U$. Then there is one federated input $X_1$ and no shared input. By Definition~\ref{def:unary-semantics}, $Y^{(c)} = u(X_1^{(c)})$ for every $c \in \clients$. Let $c \in \clients$ and let $i$ be a local index of $Y^{(c)}$. Then $\vglob{Y}[\mu_Y(c,i)] = Y^{(c)}[i] = u(X_1^{(c)}[i]) = u(\vglob{X}_1[\mu_{X_1}(c,i)])$. Hence $\vglob{Y} = u(\vglob{X}_1) = F^{\operatorname{ord}}(\widetilde A_1)$.

Assume next that $F=\beta$ for some $\beta \in \mathcal B \cup \mathcal G$ and that both inputs are federated. Let the inputs be $X_1$ and $X_2$. By Definition~\ref{def:binary-elementwise-semantics}, $Y^{(c)} = \beta(X_1^{(c)},X_2^{(c)})$ for every $c \in \clients$. Because the primitive application is typed, Definition~\ref{def:typing-binary-elementwise} implies that $X_1$ and $X_2$ have the same federated type. Hence they have the same record axis. The semantic compatibility condition in Definition~\ref{def:binary-elementwise-semantics} gives $\operatorname{shape}(X_1^{(c)})=\operatorname{shape}(X_2^{(c)})$ for every $c \in \clients$. In particular, their local record counts and client offsets agree. Let $i$ be a local index. Then $\vglob{Y}[\mu_Y(c,i)] = \beta(X_1^{(c)}[i],X_2^{(c)}[i]) = \beta(\vglob{X}_1[\mu_{X_1}(c,i)],\vglob{X}_2[\mu_{X_2}(c,i)])$. Therefore $\vglob{Y} = \beta(\vglob{X}_1,\vglob{X}_2) = F^{\operatorname{ord}}(\widetilde A_1,\widetilde A_2)$.

Assume next that $F=\beta$ for some $\beta \in \mathcal B \cup \mathcal G$, that the first input is federated, and that the second input is shared. Let the inputs be $X_1$ and $S_1$. By Definition~\ref{def:binary-elementwise-semantics}, $Y^{(c)} = \beta(X_1^{(c)},\operatorname{Br}_c(S_1;X_1))$ for every $c \in \clients$. Because the primitive application is typed, Definition~\ref{def:typing-binary-elementwise} together with Definition~\ref{def:record-agnostic-broadcast-compatibility} implies that the shared tensor $S_1$ has extent $1$ along the record axis when aligned against the symbolic shape of $X_1$. Hence the ordinary broadcast of $S_1$ against $\vglob{X}_1$ is defined and, for every client $c$ and every local index $i$, $\operatorname{Br}_{\operatorname{shape}(S_1)}^{\operatorname{shape}(\vglob{X}_1)}(S_1)[\mu_{X_1}(c,i)] = \operatorname{Br}_c(S_1;X_1)[i]$ by Definitions~\ref{def:broadcast-map} and~\ref{def:local-broadcast}. Therefore $\vglob{Y}[\mu_Y(c,i)] = Y^{(c)}[i] = \beta(X_1^{(c)}[i],\operatorname{Br}_c(S_1;X_1)[i]) = \beta(\vglob{X}_1[\mu_{X_1}(c,i)],\operatorname{Br}_{\operatorname{shape}(S_1)}^{\operatorname{shape}(\vglob{X}_1)}(S_1)[\mu_{X_1}(c,i)])$. Hence $\vglob{Y} = \beta(\vglob{X}_1,S_1) = F^{\operatorname{ord}}(\widetilde A_1,\widetilde A_2)$.

The case in which $F=\beta$, the first input is shared, and the second input is federated is identical. If the inputs are $S_1$ and $X_1$, then $\vglob{Y} = \beta(S_1,\vglob{X}_1) = F^{\operatorname{ord}}(\widetilde A_1,\widetilde A_2)$.

Assume next that $F=\alpha_j$ for some $\alpha \in \mathcal A$, and that the output is federated. Then the unique input is a federated tensor $X_1$, and Definition~\ref{def:typing-aggregation} implies that $j$ is not its record axis $r$. By Definition~\ref{def:aggregation-semantics}, $Y^{(c)} = \alpha_j(X_1^{(c)})$ for every $c \in \clients$. Let $i$ be a local index of $Y^{(c)}$. Since $j \neq r$, the fiber used to compute $Y^{(c)}[i]$ lies entirely inside the local tensor $X_1^{(c)}$. The corresponding fiber of $\vglob{X}_1$ is the same fiber. Therefore $\vglob{Y}[\mu_Y(c,i)] = \alpha_j(X_1^{(c)})[i] = \alpha_j(\vglob{X}_1)[\mu_Y(c,i)]$. Hence $\vglob{Y} = \alpha_j(\vglob{X}_1) = F^{\operatorname{ord}}(\widetilde A_1)$.

Assume next that $F=\tau$ for some permutation $\tau \in \mathfrak S_k$ and that the unique input is a federated tensor $X_1$ of ordinary order $k$ with record axis $r$. By Definition~\ref{def:permutation-semantics}, $Y^{(c)} = \tau(X_1^{(c)})$ for every $c \in \clients$, and $Y$ has record axis $\tau(r)$. The local record counts of $Y$ are the same as those of $X_1$, so $o_Y(c)=o_{X_1}(c)$ for every client $c$. Let $i=(i_1,\dots,i_k)$ be a local index of $Y^{(c)}$, and define $j=(i_{\tau(1)},\dots,i_{\tau(k)})$. By Definition~\ref{def:tensor_ops}, $Y^{(c)}[i] = X_1^{(c)}[j]$. Let $g=\mu_Y(c,i)$. Since the record axis of $Y$ is $\tau(r)$, we have $g_{\tau(t)} = i_{\tau(t)}$ for $t \neq r$ and $g_{\tau(r)} = o_Y(c)+i_{\tau(r)} = o_{X_1}(c)+j_r$. Hence $(g_{\tau(1)},\dots,g_{\tau(k)}) = \mu_{X_1}(c,j)$. Using the ordinary permutation rule of Definition~\ref{def:tensor_ops}, we obtain $\tau(\vglob{X}_1)[g] = \vglob{X}_1[g_{\tau(1)},\dots,g_{\tau(k)}] = \vglob{X}_1[\mu_{X_1}(c,j)] = X_1^{(c)}[j] = Y^{(c)}[i]$. Therefore $\vglob{Y} = \tau(\vglob{X}_1) = F^{\operatorname{ord}}(\widetilde A_1)$.

Assume next that $F=\mathsf{MatMul}_{\Fed\Sh}$. Then the inputs are a federated matrix $X_1$ with record axis $1$ and a shared matrix $S_1$. By Definition~\ref{def:fed-shared-matmul-semantics}, $Y^{(c)} = X_1^{(c)}S_1$ for every $c \in \clients$. Let $N := \sum_{c\in\clients} n_c(X_1)$. Then $\vglob{X}_1 \in \RR^{N \times p}$ and $\vglob{Y} \in \RR^{N \times q}$. Concatenation along the first axis commutes with right multiplication by the same matrix. Hence $\vglob{Y} = \vglob{X}_1S_1 = F^{\operatorname{ord}}(\widetilde A_1,\widetilde A_2)$.

Assume finally that $F=\mathsf{MatMul}_{\Sh\Fed}$. Then the inputs are a shared matrix $S_1$ and a federated matrix $X_1$ with record axis $2$. By Definition~\ref{def:shared-fed-matmul-semantics}, $Y^{(c)} = S_1X_1^{(c)}$ for every $c \in \clients$. Concatenation along the second axis commutes with left multiplication by the same matrix. Therefore $\vglob{Y} = S_1\vglob{X}_1 = F^{\operatorname{ord}}(\widetilde A_1,\widetilde A_2)$.

This completes the case in which $Y$ is federated.

We now treat the case in which $Y$ is shared.

If no input is federated, then every input is shared. In that case the only possible primitive applications in the base language are shared-only applications of symbols from $\mathcal U$, $\mathcal B$, $\mathcal G$, $\{\alpha_j : \alpha \in \mathcal A,\ j \geq 1\}$, and $\mathcal P$. By Definitions~\ref{def:unary-semantics}, \ref{def:binary-elementwise-semantics}, \ref{def:aggregation-semantics}, and~\ref{def:permutation-semantics}, their shared semantics agrees with the corresponding ordinary operation of Definition~\ref{def:tensor_ops}. Therefore $Y = F^{\operatorname{ord}}(\widetilde A_1,\dots,\widetilde A_m)$.

Assume next that at least one input is federated. Then Proposition~\ref{prop:exposure-discipline} shows that exactly one of two cases occurs.

In the first case, $Y = \alpha_r(X_1)$ for some federated tensor $X_1$ with record axis $r$. By Definition~\ref{def:aggregation-semantics}, $Y = \alpha_r(\vglob{X}_1) = F^{\operatorname{ord}}(\widetilde A_1)$.

In the second case, $Y = \mathsf{MatMul}_{\Fed\Fed}(X_1,X_2)$ for federated matrices $X_1$ and $X_2$ with local shapes $X_1^{(c)} \in \RR^{a \times n_c}$ and $X_2^{(c)} \in \RR^{n_c \times b}$ for every $c \in \clients$. Let $N := \sum_{c\in\clients} n_c$. Then $\vglob{X}_1 \in \RR^{a \times N}$ and $\vglob{X}_2 \in \RR^{N \times b}$. By Definition~\ref{def:fed-fed-matmul-semantics}, $Y = \sum_{c\in\clients} X_1^{(c)}X_2^{(c)}$. The matrix $\vglob{X}_1$ is the horizontal concatenation of the blocks $X_1^{(c)}$, and the matrix $\vglob{X}_2$ is the vertical concatenation of the blocks $X_2^{(c)}$ in the same client order. Therefore block multiplication gives $\vglob{X}_1\vglob{X}_2 = \sum_{c\in\clients} X_1^{(c)}X_2^{(c)}$. Hence $Y = \vglob{X}_1\vglob{X}_2 = F^{\operatorname{ord}}(\widetilde A_1,\widetilde A_2)$.

These cases exhaust the primitive applications of symbols in $\Sigma_0$. Therefore, if $Y$ is a federated tensor, then $\vglob{Y} = F^{\operatorname{ord}}(\widetilde A_1,\dots,\widetilde A_m)$, and if $Y$ is a shared tensor, then $Y = F^{\operatorname{ord}}(\widetilde A_1,\dots,\widetilde A_m)$. This proves the proposition.
\end{proof}

\section{Proofs for shared-state factorizations}
\label{app:shared-state-proofs}

\textbf{Merge structures.}
Repeated application of the concatenation law in Section~\ref{subsec:one-round-typed-programs} gives the client-partition merge identity, including empty client blocks. The identity notation $0_i$ denotes the monoid identity: for product aggregation it is the all-ones tensor. In contrast, maximum on unrestricted real tensors has no empty-input identity; using a bottom element requires a restriction or enlargement of the state domain. Raw concatenation changes shape and depends on order. An expression that concatenates inputs and then reduces the result belongs to the factorization class only through a valid statistic of count-independent shape. These restrictions concern the shared-output boundaries in the program class, rather than shared-only or non-$r$ aggregations elsewhere in the language.

\begin{proof}[Proof of Theorem~\ref{thm:one-round-shared-state-factorization}]
Let $\sigma$ be induced by a one-round typed program of input type $\Fed_r(\mathbf d)$ and output shape $\mathbf t$. Let $g_1(x),\dots,g_q(x)$ be its shared-state expressions, and let
\[
y_1:\Sh(\mathbf s_1),\dots,y_q:\Sh(\mathbf s_q)
\vdash
h(y_1,\dots,y_q):\Sh(\mathbf t)
\]
be its shared-only decoder expression.

We first record a consequence of client-locality for expressions. Let $e(x)$ be a client-local federated expression with one free variable $x:\Fed_r(\mathbf d)$. We claim that, for each $n \in \mathbb N_0$, there is a deterministic local map $\widehat e_n$ defined on $\mathcal X_n(r,\mathbf d)$ such that, for every finite ordered federation $\clients$, every federated tensor $X=\{X^{(c)}\}_{c \in \clients}$ of type $\Fed_r(\mathbf d)$, and every client $c \in \clients$, one has $e(X)^{(c)} = \widehat e_{n_c}(X^{(c)})$. This claim follows by induction on the syntax of $e$. If $e(x)=x$, take $\widehat e_n$ to be the identity map on $\mathcal X_n(r,\mathbf d)$. For the induction step, suppose that $e$ is obtained by applying a primitive symbol $F$ to subexpressions that have already been given local realizations. Since $e$ is client-local and has federated type, every primitive subexpression of $e$ that receives a federated input has federated output. Proposition~\ref{prop:client-locality} then gives a deterministic client-local map for the final primitive application. Composing this map with the local realizations of the subexpressions gives the required map $\widehat e_n$. This proves the claim.

We now construct the shared-state factorization. For each component $i \in [q]$, define a tensor space $M_i := \RR^{\mathbf s_i}$ and a commutative monoid operation $\oplus_i$ on $M_i$ as follows.

If $g_i$ has the record-axis aggregation form, then by Definition~\ref{def:one-round-typed-program} there are a client-local federated expression $e_i(x)$, a type $\Fed_{r_i}(\mathbf u_i)$, and an aggregation schema $\alpha^{(i)} \in \mathcal A$ such that $g_i(x)=\alpha^{(i)}_{r_i}(e_i(x))$. The same definition supplies a commutative monoid $(\RR^{\mathbf s_i},\oplus_i,0_i)$ satisfying the required client-wise merge identity for $\alpha^{(i)}_{r_i}$.

If $g_i$ has the federated-federated matrix-product form, then $g_i(x)=\mathsf{MatMul}_{\Fed\Fed}(a_i(x),b_i(x))$ for client-local federated expressions $a_i(x)$ and $b_i(x)$. In this case $M_i=\RR^{m_i \times n_i}$, and we take $\oplus_i$ to be ordinary addition and $0_i$ to be the zero matrix.

Define $M := M_1 \times \ldots \times M_q$. Define $\oplus_M : M \times M \to M$ componentwise by
\[
(z_1,\dots,z_q) \oplus_M (z'_1,\dots,z'_q)
:=
(z_1 \oplus_1 z'_1,\dots,z_q \oplus_q z'_q),
\]
and define $0_M := (0_1,\dots,0_q)$. Since each $(M_i,\oplus_i,0_i)$ is a commutative monoid, $(M,\oplus_M,0_M)$ is a commutative monoid. Also, $M$ is a finite product of tensor spaces.

For each $n \in \mathbb N_0$, define the local encoder $\phi_n:\mathcal X_n(r,\mathbf d)\to M$ componentwise. Let $Z \in \mathcal X_n(r,\mathbf d)$.

If $g_i$ has the record-axis aggregation form, let $\widehat e_{i,n}$ be the local realization of $e_i$ constructed above. Define $\phi_{i,n}(Z):=\alpha^{(i)}_{r_i}(\widehat e_{i,n}(Z))$.

If $g_i$ has the federated-federated matrix-product form, let $\widehat a_{i,n}$ and $\widehat b_{i,n}$ be the local realizations of $a_i$ and $b_i$. Define $\phi_{i,n}(Z):=\widehat a_{i,n}(Z)\widehat b_{i,n}(Z)$.

Set $\phi_n(Z) := (\phi_{1,n}(Z),\dots,\phi_{q,n}(Z))$.

Define the decoder $\psi:M\to\RR^{\mathbf t}$ by evaluating the shared-only expression $h$ on its arguments. Thus, $\psi(z_1,\dots,z_q) := h(z_1,\dots,z_q)$.

It remains to prove the factorization identity. Let $\clients$ be a finite ordered federation, and let $X=\{X^{(c)}\}_{c\in\clients}$ be a federated tensor of type $\Fed_r(\mathbf d)$ with local record counts $\{n_c\}_{c\in\clients}$. We show that, for each $i \in [q]$,
\[
g_i(X)
=
\bigoplus_{c\in\clients} \phi_{i,n_c}(X^{(c)}).
\]

If $g_i$ has the record-axis aggregation form, then $g_i(X)=\alpha^{(i)}_{r_i}(e_i(X))$. By the local-realization claim, $e_i(X)^{(c)}=\widehat e_{i,n_c}(X^{(c)})$ for every $c \in \clients$. By the merge identity in Definition~\ref{def:one-round-typed-program},
\[
g_i(X)
=
\bigoplus_{c\in\clients}
\alpha^{(i)}_{r_i}(e_i(X)^{(c)})
=
\bigoplus_{c\in\clients}
\alpha^{(i)}_{r_i}(\widehat e_{i,n_c}(X^{(c)}))
=
\bigoplus_{c\in\clients} \phi_{i,n_c}(X^{(c)}).
\]

If $g_i$ has the federated-federated matrix-product form, then by Definition~\ref{def:fed-fed-matmul-semantics},
\[
g_i(X)
=
\sum_{c\in\clients} a_i(X)^{(c)} b_i(X)^{(c)}.
\]
By the local-realization claim applied to $a_i$ and $b_i$, one has $a_i(X)^{(c)}=\widehat a_{i,n_c}(X^{(c)})$ and $b_i(X)^{(c)}=\widehat b_{i,n_c}(X^{(c)})$ for every $c \in \clients$. Hence
\[
g_i(X)
=
\sum_{c\in\clients}
\widehat a_{i,n_c}(X^{(c)})\widehat b_{i,n_c}(X^{(c)})
=
\bigoplus_{c\in\clients} \phi_{i,n_c}(X^{(c)}),
\]
where the last merge is ordinary addition in $M_i$.

Therefore
\[
(g_1(X),\dots,g_q(X))
=
\bigoplus_{c\in\clients} \phi_{n_c}(X^{(c)})
\]
in $M$. Applying the decoder gives
\[
\sigma_{\clients}(X)
=
h(g_1(X),\dots,g_q(X))
=
\psi\left(\bigoplus_{c\in\clients} \phi_{n_c}(X^{(c)})\right).
\]
This is the identity required in Definition~\ref{def:shared-state-factorization}. Hence $\sigma$ admits a shared-state factorization.
\end{proof}

\begin{proof}[Proof of Proposition~\ref{prop:effective-factorization-extraction}]
Represent the derivation $\mathcal D$ by a directed acyclic graph whose nodes carry typing and primitive metadata. Its size $|\mathcal D|$ counts nodes, edges and annotations in this representation. Merge structures are given semantic data. Shape and structure data are retained by symbolic references, with constant-time primitive metadata lookup.

Traverse the graph in dependency order, memoizing symbolic local realizations in the encoder subgraphs and retaining the shared-only decoder graph. For the input variable, emit the identity map. For a client-local primitive application, emit its local evaluation rule with references to the already compiled operands. Proposition~\ref{prop:client-locality} justifies this composition. The result is one family $\{\widehat e_n\}_{n\in\mathbb N_0}$ generated by the same expression and primitive rules, with $n$ supplied by the local input shape. It contains no client-identity parameter or access to another client's data; its evaluation cost may grow with $n$.

For each declared boundary $g_i$, emit one component encoder. In the aggregation case emit $\phi_{i,n}=\alpha^{(i)}_{r_i}\circ\widehat e_{i,n}$ and retain its valid merge and identity. In the contraction case emit
\[
\phi_{i,n}(Z)=\widehat a_{i,n}(Z)\widehat b_{i,n}(Z)
\]
with addition and the zero matrix. The program's count compatibility condition ensures that this product is defined. Form the symbolic product state $M=\prod_{i=1}^q\RR^{\mathbf s_i}$, the componentwise merge, the tuple encoder $\phi_n=(\phi_{1,n},\ldots,\phi_{q,n})$, and the decoder $\psi=h$. The proof of Theorem~\ref{thm:one-round-shared-state-factorization} verifies the factorization identity.

Each node and edge is visited once and emits a bounded number of symbolic instructions or references. The output retains graph sharing, including in the decoder, so its time and space costs are $O(|\mathcal D|)$. Local-node memoization does not remove declared boundaries: repeated components of the tuple remain present.
\end{proof}

The $O(|\mathcal D|)$ bound concerns symbolic extraction from the provided valid derivation, without numeric tensor evaluation, type-checking, proof of merge laws, expansion of shared subexpressions into trees, enumeration of input counts, or dimension-sized array allocation. Arithmetic on arbitrarily large shape integers is also outside the constant-time metadata convention.

\begin{proof}[Proof of Corollary~\ref{cor:extracted-state-dimension}]
A tensor of shape $\mathbf s=(s_1,\ldots,s_k)$ has $\nu(\mathbf s)=\prod_{j=1}^k s_j$ real coordinates, with one coordinate when $k=0$. Flattening each component identifies $M=\prod_{i=1}^q\RR^{\mathbf s_i}$ with a real vector space of dimension $\sum_{i=1}^q\nu(\mathbf s_i)$. Components are counted according to the declared tuple. Applying the same argument in round $t$ gives $D_t(P)=\sum_{i=1}^{q_t}\nu(\mathbf s_{t,i})$. Appendix~\ref{app:complexity} gives the associated evaluation-cost accounting.
\end{proof}

\begin{definition}[Typed-realizable shared-state factorization]
\label{def:typed-realizable-factorization}
A shared-state factorization of a computation $\sigma$ with input type $\Fed_r(\mathbf d)$ and output shape $\mathbf t$ is \emph{typed-realizable} if it has state space $M=\prod_{i=1}^q\RR^{\mathbf s_i}$, componentwise commutative monoid merges $(\RR^{\mathbf s_i},\oplus_i,0_i)$, and the following witnesses over $\Sigma_0$.

Each component encoder has one of two forms. For an aggregation-realized component, there are a client-local expression $e_i(x)$ and an aggregation schema $\alpha^{(i)}\in\mathcal A$ such that
\[
x:\Fed_r(\mathbf d)\vdash e_i(x):\Fed_{r_i}(\mathbf u_i),
~
\phi_{i,n}(Z)=\alpha^{(i)}_{r_i}(\widehat e_{i,n}(Z))
\]
for all $n\in\mathbb N_0$ and $Z\in\mathcal X_n(r,\mathbf d)$. The output shape is $\mathbf s_i$, and $\alpha^{(i)}_{r_i}$ has the merge structure $(\RR^{\mathbf s_i},\oplus_i,0_i)$ of Section~\ref{subsec:one-round-typed-programs}.

For a matrix-product-realized component, there are client-local expressions $a_i(x),b_i(x)$ and positive integers $m_i,n_i$ with
\[
x:\Fed_r(\mathbf d)\vdash a_i(x):\Fed_2((m_i)),
~
x:\Fed_r(\mathbf d)\vdash b_i(x):\Fed_1((n_i)).
\]
The application $\mathsf{MatMul}_{\Fed\Fed}(a_i(x),b_i(x))$ satisfies the client-wise count compatibility condition of Definition~\ref{def:fed-fed-matmul-semantics} for every input. In this case $\mathbf s_i=(m_i,n_i)$, the merge is addition with zero matrix identity, and
\[
\phi_{i,n}(Z)=\widehat a_{i,n}(Z)\widehat b_{i,n}(Z),
\]
where $n\in\mathbb N_0,~Z\in\mathcal X_n(r,\mathbf d)$. The maps $\widehat e_{i,n},\widehat a_{i,n},\widehat b_{i,n}$ are the uniform local families generated from their respective expressions by Proposition~\ref{prop:client-locality} and structural induction.

Finally, there is a shared-only decoder expression
\[
y_1:\Sh(\mathbf s_1),\ldots,y_q:\Sh(\mathbf s_q)
\vdash h(y_1,\ldots,y_q):\Sh(\mathbf t)
\]
that is defined on all of $M$ and satisfies $\psi(z_1,\ldots,z_q)=h(z_1,\ldots,z_q)$ for every $(z_1,\ldots,z_q)\in M$.
\end{definition}

A partial operation, such as a linear solve, does not provide a total map on $M$ through shared-only typing. Its use therefore requires either a total extension or a separate domain-qualified formulation of the factorization.

\begin{proof}[Proof of Theorem~\ref{thm:one-round-realization}]
Let $\sigma$ have a typed-realizable shared-state factorization according to Definition~\ref{def:typed-realizable-factorization}. We construct a one-round typed program representing $\sigma$.

For each component $i \in [q]$, define a shared-state expression $g_i(x)$ as follows. If the component is aggregation-realized, set $g_i(x):=\alpha^{(i)}_{r_i}(e_i(x))$. If the component is matrix-product-realized, set $g_i(x):=\mathsf{MatMul}_{\Fed\Fed}(a_i(x),b_i(x))$. In both cases, $g_i(x)$ has shared type $\Sh(\mathbf s_i)$ by the typing rules of Appendix~\ref{app:language-spec}. The decoder expression $h$, merge structure and contraction compatibility are given by Definition~\ref{def:typed-realizable-factorization}. These choices form a one-round typed program (see Definition~\ref{def:one-round-typed-program}).

It remains to show that this program induces $\sigma$. Let $\clients$ be a finite ordered federation, and let $X=\{X^{(c)}\}_{c\in\clients}$ be a federated tensor of type $\Fed_r(\mathbf d)$ with local record counts $\{n_c\}_{c\in\clients}$.

Fix $i \in [q]$. If the component is aggregation-realized, then $g_i(X)=\alpha^{(i)}_{r_i}(e_i(X))$. By the assumed merge identity for $\alpha^{(i)}_{r_i}$ and by the local realization of $e_i$, one has
\[
g_i(X)
=
\bigoplus_{c\in\clients}\alpha^{(i)}_{r_i}(e_i(X)^{(c)})
=
\bigoplus_{c\in\clients}\alpha^{(i)}_{r_i}(\widehat e_{i,n_c}(X^{(c)}))
=
\bigoplus_{c\in\clients}\phi_{i,n_c}(X^{(c)}).
\]

If the component is matrix-product-realized, then $g_i(X)=\mathsf{MatMul}_{\Fed\Fed}(a_i(X),b_i(X))$. By Definition~\ref{def:fed-fed-matmul-semantics} and by the local realizations of $a_i$ and $b_i$, one has
\[
g_i(X)
=
\sum_{c\in\clients} a_i(X)^{(c)}b_i(X)^{(c)}
=
\sum_{c\in\clients} \widehat a_{i,n_c}(X^{(c)})\widehat b_{i,n_c}(X^{(c)})
=
\bigoplus_{c\in\clients}\phi_{i,n_c}(X^{(c)}),
\]
where the last merge is ordinary addition.

Thus, for every $i \in [q]$, $g_i(X)=\bigoplus_{c\in\clients}\phi_{i,n_c}(X^{(c)})$. Since the merge on $M$ is componentwise, this gives
\[
(g_1(X),\dots,g_q(X))
=
\bigoplus_{c\in\clients}\phi_{n_c}(X^{(c)}).
\]
Using the decoder identity $\psi(z_1,\dots,z_q)=h(z_1,\dots,z_q)$, we obtain
\[
h(g_1(X),\dots,g_q(X))
=
\psi\left(\bigoplus_{c\in\clients}\phi_{n_c}(X^{(c)})\right).
\]
By the shared-state factorization identity of Definition~\ref{def:shared-state-factorization}, the right-hand side is $\sigma_{\clients}(X)$. Therefore the constructed one-round typed program induces $\sigma$. This proves the theorem.
\end{proof}

\begin{proof}[Proof of Theorem~\ref{thm:iterative-shared-state-factorization}]
Fix a round $t \in \{0,\dots,T-1\}$. Let $g_{t,1}(x,\theta),\dots,g_{t,q_t}(x,\theta)$ be the shared-state expressions of round $t$, and let
\[
y_1:\Sh(\mathbf s_{t,1}),\dots,y_{q_t}:\Sh(\mathbf s_{t,q_t}),\theta:\Sh(\boldsymbol{\tau}_t)
\vdash
h_t(y_1,\dots,y_{q_t},\theta):\Sh(\boldsymbol{\tau}_{t+1})
\]
be the round-$t$ shared-only decoder expression.

We first record a client-local realization claim for round-$t$ expressions. Let $e(x,\theta)$ be a client-local federated expression with typing judgment
\[
x:\Fed_r(\mathbf d),\theta:\Sh(\boldsymbol{\tau}_t)
\vdash
e(x,\theta):\Fed_{r'}(\mathbf u).
\]
For each fixed $\theta \in \RR^{\boldsymbol{\tau}_t}$ and each $n \in \mathbb N_0$, there is a deterministic local map $\widehat e_n(\cdot;\theta)$ on $\mathcal X_n(r,\mathbf d)$ such that, for every finite ordered federation $\clients$, every federated tensor $X=\{X^{(c)}\}_{c\in\clients}$ of type $\Fed_r(\mathbf d)$, and every client $c \in \clients$, one has $e(X,\theta)^{(c)}=\widehat e_{n_c}(X^{(c)};\theta)$. The proof is by induction on the syntax of $e$. If $e(x,\theta)=x$, the claim holds with the identity map. In the induction step, $e$ is obtained by applying a primitive symbol to subexpressions that already have local realizations. Since $e$ is client-local, every primitive subexpression that receives a federated input has federated output. Proposition~\ref{prop:client-locality} gives a deterministic client-local map for that primitive application, with $\theta$ treated as a shared input. Composing this map with the local realizations of the subexpressions gives $\widehat e_n(\cdot;\theta)$.

We now construct the round-$t$ shared-state factorization. For each $i \in [q_t]$, define $M_{t,i}:=\RR^{\mathbf s_{t,i}}$ and define a commutative monoid operation $\oplus_{t,i}$ on $M_{t,i}$ as follows.

If $g_{t,i}$ has the aggregation form, then Definition~\ref{def:iterative-typed-program} gives a client-local federated expression $e_{t,i}(x,\theta)$, a type $\Fed_{r_{t,i}}(\mathbf u_{t,i})$, and an aggregation schema $\alpha^{(t,i)} \in \mathcal A$ such that $g_{t,i}(x,\theta)=\alpha^{(t,i)}_{r_{t,i}}(e_{t,i}(x,\theta))$. The same definition supplies a commutative monoid $(\RR^{\mathbf s_{t,i}},\oplus_{t,i},0_{t,i})$ satisfying the client-wise merge identity for $\alpha^{(t,i)}_{r_{t,i}}$.

If $g_{t,i}$ has the matrix-product form, then $g_{t,i}(x,\theta)=\mathsf{MatMul}_{\Fed\Fed}(a_{t,i}(x,\theta),b_{t,i}(x,\theta))$ for client-local federated expressions $a_{t,i}$ and $b_{t,i}$. In this case, $M_{t,i}=\RR^{m_{t,i}\times n_{t,i}}$, and we take $\oplus_{t,i}$ to be ordinary addition and $0_{t,i}$ to be the zero matrix.

Define $M_t:=M_{t,1}\times\cdots\times M_{t,q_t}$. Define $\oplus_t:M_t\times M_t\to M_t$ componentwise by
\[
(z_1,\dots,z_{q_t})\oplus_t(z'_1,\dots,z'_{q_t})
:=
(z_1\oplus_{t,1}z'_1,\dots,z_{q_t}\oplus_{t,q_t}z'_{q_t}),
\]
and define $0_t:=(0_{t,1},\dots,0_{t,q_t})$. Since each component is a commutative monoid, $(M_t,\oplus_t,0_t)$ is a commutative monoid.

For each $n \in \mathbb N_0$ and each $\theta \in \RR^{\boldsymbol{\tau}_t}$, define the local encoder $\phi_{t,n}(\cdot;\theta):\mathcal X_n(r,\mathbf d)\to M_t$ componentwise. Let $Z \in \mathcal X_n(r,\mathbf d)$.

If $g_{t,i}$ has the aggregation form, let $\widehat e_{t,i,n}(\cdot;\theta)$ be the local realization of $e_{t,i}(x,\theta)$ constructed above. Define
\[
\phi_{t,i,n}(Z;\theta)
:=
\alpha^{(t,i)}_{r_{t,i}}(\widehat e_{t,i,n}(Z;\theta)).
\]

If $g_{t,i}$ has the matrix-product form, let $\widehat a_{t,i,n}(\cdot;\theta)$ and $\widehat b_{t,i,n}(\cdot;\theta)$ be the local realizations of $a_{t,i}(x,\theta)$ and $b_{t,i}(x,\theta)$. Define
\[
\phi_{t,i,n}(Z;\theta)
:=
\widehat a_{t,i,n}(Z;\theta)\widehat b_{t,i,n}(Z;\theta).
\]

Set
\[
\phi_{t,n}(Z;\theta)
:=
(\phi_{t,1,n}(Z;\theta),\dots,\phi_{t,q_t,n}(Z;\theta)).
\]

Define the decoder $\psi_t:M_t\times\RR^{\boldsymbol{\tau}_t}\to\RR^{\boldsymbol{\tau}_{t+1}}$ by
\[
\psi_t((z_1,\dots,z_{q_t}),\theta)
:=
h_t(z_1,\dots,z_{q_t},\theta).
\]

It remains to prove the factorization identity. Let $\clients$ be a finite ordered federation, let $X=\{X^{(c)}\}_{c\in\clients}$ be a federated tensor of type $\Fed_r(\mathbf d)$ with local record counts $\{n_c\}_{c\in\clients}$, and let $\theta \in \RR^{\boldsymbol{\tau}_t}$.

We show that, for each $i \in [q_t]$,
\[
g_{t,i}(X,\theta)
=
\bigoplus_{c\in\clients}\phi_{t,i,n_c}(X^{(c)};\theta).
\]

If $g_{t,i}$ has the aggregation form, then $g_{t,i}(X,\theta)=\alpha^{(t,i)}_{r_{t,i}}(e_{t,i}(X,\theta))$. By the local-realization claim, $e_{t,i}(X,\theta)^{(c)}=\widehat e_{t,i,n_c}(X^{(c)};\theta)$ for every $c \in \clients$. By the merge identity in Definition~\ref{def:iterative-typed-program},
\[
g_{t,i}(X,\theta)
=
\bigoplus_{c\in\clients}
\alpha^{(t,i)}_{r_{t,i}}(e_{t,i}(X,\theta)^{(c)})
=
\bigoplus_{c\in\clients}
\alpha^{(t,i)}_{r_{t,i}}(\widehat e_{t,i,n_c}(X^{(c)};\theta))
=
\bigoplus_{c\in\clients}\phi_{t,i,n_c}(X^{(c)};\theta).
\]

If $g_{t,i}$ has the matrix-product form, then by Definition~\ref{def:fed-fed-matmul-semantics},
\[
g_{t,i}(X,\theta)
=
\sum_{c\in\clients} a_{t,i}(X,\theta)^{(c)} b_{t,i}(X,\theta)^{(c)}.
\]
By the local-realization claim applied to $a_{t,i}$ and $b_{t,i}$, one has $a_{t,i}(X,\theta)^{(c)}=\widehat a_{t,i,n_c}(X^{(c)};\theta)$ and $b_{t,i}(X,\theta)^{(c)}=\widehat b_{t,i,n_c}(X^{(c)};\theta)$ for every $c \in \clients$. Hence
\[
g_{t,i}(X,\theta)
=
\sum_{c\in\clients}
\widehat a_{t,i,n_c}(X^{(c)};\theta)\widehat b_{t,i,n_c}(X^{(c)};\theta)
=
\bigoplus_{c\in\clients}\phi_{t,i,n_c}(X^{(c)};\theta),
\]
where the last merge is ordinary addition in $M_{t,i}$.

Therefore
\[
(g_{t,1}(X,\theta),\dots,g_{t,q_t}(X,\theta))
=
\bigoplus_{c\in\clients}\phi_{t,n_c}(X^{(c)};\theta)
\]
in $M_t$. Applying $\psi_t$ gives
\[
U_{t,\clients}(X,\theta)
=
h_t(g_{t,1}(X,\theta),\dots,g_{t,q_t}(X,\theta),\theta)
=
\psi_t\left(\bigoplus_{c\in\clients}\phi_{t,n_c}(X^{(c)};\theta),\theta\right).
\]
This proves the theorem.
\end{proof}

\section{Privacy integration}
\label{app:privacy-integration}

The shared-state factorization $\sigma_{\clients}(X)=\psi\left(\bigoplus_{c \in \clients} \phi_{n_c}(X^{(c)})\right)$ identifies three integration points for privacy mechanisms: the local encoder $\phi$, the merge operation $\bigoplus$, and the decoder $\psi$. The mechanisms discussed here are not part of the deterministic base signature $\Sigma_0$. They define randomized or cryptographic realizations of the same factorization interface.

The following theorem formalizes the compatibility of shared-state factorizations with differential privacy. It is a lifting result; once the deterministic computation is expressed through the shared-state query, privacy mechanisms can be applied either to the local messages, to the merged shared state, or to the decoded output.

\begin{theorem}[Privacy lifting for shared-state factorizations]
\label{thm:privacy-lifting}
Let $\sigma$ admit a shared-state factorization with state space $M$, merge operation $\oplus_M$, local encoders $\{\phi_n\}_{n\in\mathbb N_0}$, and decoder $\psi$. For a finite ordered federation $\clients$, define the deterministic shared-state query $Q_{\clients}(X):=\bigoplus_{c\in\clients}\phi_{n_c}(X^{(c)})$.

Equip the tensor spaces and their finite products with their Borel sigma-algebras. Assume that the encoders, merge, and decoder are measurable on their domains. Any restricted decoder domain is a Borel subset of $M$, and every randomized merged state in this appendix lies in that domain with probability one. Randomized mechanisms in this appendix are measurable probability kernels.

Let $\sim$ be an adjacency relation on federated inputs of the same type. Equip $M$ with a metric $d_M$, and suppose that
\[
\Delta_Q
:=
\sup_{X\sim X'} d_M(Q_{\clients}(X),Q_{\clients}(X'))
<
\infty .
\]
Let $\mathcal R_M:M\leadsto M$ be a randomized mechanism such that, for all $m,m'\in M$ with $d_M(m,m')\leq \Delta_Q$ and all measurable $E\subseteq M$,
\[
\Pr[\mathcal R_M(m)\in E]
\leq
e^\varepsilon \Pr[\mathcal R_M(m')\in E]+\delta .
\]
Then the randomized release $\mathcal M_{\clients}(X):=\psi(\mathcal R_M(Q_{\clients}(X)))$ is $(\varepsilon,\delta)$-differentially private with respect to $\sim$.

Moreover, let $\sim_n$ be a local adjacency relation on $\mathcal X_n(r,\mathbf d)$ for each $n\in\mathbb N_0$. Suppose that, for each $n$, $\mathcal R_n:\mathcal X_n(r,\mathbf d)\leadsto M$ is an $\varepsilon$-local-DP randomizer with respect to $\sim_n$. Define
\[
\mathcal M_{\clients}^{\mathrm{loc}}(X)
:=
\psi\left(\bigoplus_{c\in\clients}\mathcal R_{n_c}(X^{(c)})\right),
\]
where the randomizers are applied independently across clients. Then $\mathcal M_{\clients}^{\mathrm{loc}}$ is $\varepsilon$-private for replacing one client's contribution by a locally adjacent contribution under $\sim_{n_c}$, with the same client set, local sample counts, and all other client tensors unchanged.
\end{theorem}

\begin{proof}
We first prove the central-DP statement. Let $X\sim X'$ be adjacent federated inputs. By definition of $\Delta_Q$, one has $d_M(Q_{\clients}(X),Q_{\clients}(X'))\leq \Delta_Q$. Therefore, for every measurable $E\subseteq M$,
\[
\Pr[\mathcal R_M(Q_{\clients}(X))\in E]
\leq
e^\varepsilon
\Pr[\mathcal R_M(Q_{\clients}(X'))\in E]
+
\delta .
\]
The released value $\mathcal M_{\clients}(X)=\psi(\mathcal R_M(Q_{\clients}(X)))$ is a post-processing of $\mathcal R_M(Q_{\clients}(X))$. For a restricted domain, extend $\psi$ by a constant outside it; this is measurable and leaves the release law unchanged. Differential privacy is closed under post-processing. Hence $\mathcal M_{\clients}$ is $(\varepsilon,\delta)$-differentially private with respect to $\sim$.

We now prove the local-DP statement. Fix a client $c_0\in\clients$. Let $X$ and $X'$ be federated inputs that agree on every client $c\neq c_0$, and suppose that $X^{(c_0)}\sim_{n_{c_0}} X'^{(c_0)}$. For clients $c\neq c_0$, the random messages $\mathcal R_{n_c}(X^{(c)})$ and $\mathcal R_{n_c}(X'^{(c)})$ have the same distribution. For client $c_0$, the $\varepsilon$-local-DP assumption gives the required likelihood-ratio bound for the randomized message. The merge operation and the decoder are post-processing of the collection of randomized client messages. Therefore the final released value $\mathcal M_{\clients}^{\mathrm{loc}}(X)$ satisfies the same $\varepsilon$ privacy bound for a change in the local contribution of client $c_0$. This proves the theorem.
\end{proof}

Let $\Pi_{\oplus}$ be an SMPC protocol that correctly realizes the finite merge operation induced by the commutative monoid $(M,\oplus_M,0_M)$~\citep{yao1982protocols,goldreich2019play,evans2018pragmatic,smpc}. Replacing the merge $\bigoplus_{c\in\clients}$ with $\Pi_{\oplus}$ leaves the encoder $\phi$ and decoder $\psi$ unchanged, and computes the same merged state $m^\star=\bigoplus_{c\in\clients}\phi_{n_c}(X^{(c)})$. Under the protocol's threat model, its view reveals no additional information about individual summaries beyond what is implied by the ideal functionality's released values and the adversary's allowed inputs. If the functionality releases $m^\star$, this guarantee allows information implied by that state, even when $\psi(m^\star)$ reveals less.

Theorem~\ref{thm:privacy-lifting} also clarifies how local differential privacy can be integrated. If $\mathcal R_{\mathrm{LDP}}:M\leadsto M$ is an $\varepsilon$-local randomizer on the state space $M$~\citep{kasiviswanathan2011can}, one may define randomized local messages by applying $\mathcal R_{\mathrm{LDP}}$ to the encoded client summaries. The same merge operation $\oplus_M$ can be applied because the randomized messages still lie in $M$. The commutative monoid structure is not changed by this randomization, but an arbitrary local randomizer does not need to preserve the original deterministic factorization identity. It instead defines a randomized private analogue of the original shared-state computation.

Central differential privacy can be integrated either at the shared-state level or at the final-output level~\citep{dwork2006calibrating,dwork2014algorithmic}. At the shared-state level, one applies a mechanism $\mathcal R_{\mathrm{CDP}}^M$ to the merged state $m^\star$ and releases or decodes $\psi(\mathcal R_{\mathrm{CDP}}^M(m^\star))$. At the final-output level, one applies a mechanism $\mathcal R_{\mathrm{CDP}}^{\mathrm{out}}$ to $\psi(m^\star)$. In both cases, the privacy guarantee applies to the quantity that is released. If the unnoised shared state $m^\star$ is revealed, then subsequent randomized post-processing cannot make that earlier release private.

These results show formal compatibility between shared-state factorizations and standard privacy mechanisms. An analysis of privacy parameters, sensitivities, cryptographic assumptions, and protocol overhead is beyond the scope of this paper. The local result concerns replacement adjacency at unchanged sample counts. It supplies no privacy bound for sample addition or removal, repeated rounds, or randomized participation without further assumptions and analysis.

\section{Closure under shared-only post-processing}
\label{app:shared-postprocessing}

\begin{corollary}[Closure under shared-only post-processing]
\label{cor:one-round-shared-postprocessing}
Let $\sigma_1,\dots,\sigma_\ell$ be shared-output computations with common input type $\Fed_r(\mathbf d)$. Assume that each $\sigma_j$ is induced by a one-round typed program and has output shape $\mathbf t_j$. Let $z_1:\Sh(\mathbf t_1),\dots,z_\ell:\Sh(\mathbf t_\ell) \vdash H(z_1,\dots,z_\ell):\Sh(\mathbf t)$ be a shared-only expression. Define $\sigma_{\clients}(X) := H(\sigma_{1,\clients}(X),\dots,\sigma_{\ell,\clients}(X))$ for every finite ordered federation $\clients$ and every federated input $X$ of type $\Fed_r(\mathbf d)$. Then $\sigma$ is induced by a one-round typed program. In particular, $\sigma$ admits a shared-state factorization.
\end{corollary}

\begin{proof}[Proof of Corollary~\ref{cor:one-round-shared-postprocessing}]
For each $j \in [\ell]$, let the one-round typed program inducing $\sigma_j$ have shared-state expressions $g_{j,1}(x),\dots,g_{j,q_j}(x)$ with output shapes $\mathbf s_{j,1},\dots,\mathbf s_{j,q_j}$, and let its shared-only decoder be
\[
y_{j,1}:\Sh(\mathbf s_{j,1}),\dots,y_{j,q_j}:\Sh(\mathbf s_{j,q_j})
\vdash
h_j(y_{j,1},\dots,y_{j,q_j}):\Sh(\mathbf t_j).
\]
Construct a new one-round typed program by taking all shared-state expressions from all the programs. Thus, the shared-state expressions of the new program are $\{g_{j,a}(x) : j \in [\ell],\ a \in [q_j]\}$. Each of these expressions has one of the two forms allowed in Definition~\ref{def:one-round-typed-program}, because it came from one of the original one-round typed programs.

Define the new shared-only decoder by
\[
\widetilde h
:=
H\left(
h_1(y_{1,1},\dots,y_{1,q_1}),
\dots,
h_\ell(y_{\ell,1},\dots,y_{\ell,q_\ell})
\right).
\]
This is a shared-only expression because each $h_j$ is shared-only and $H$ is shared-only.

Let $\clients$ be a finite ordered federation and let $X$ be a federated input of type $\Fed_r(\mathbf d)$. The value induced by the constructed program is
\[
\widetilde h\bigl(g_{1,1}(X),\dots,g_{\ell,q_\ell}(X)\bigr)
=
H\left(
h_1(g_{1,1}(X),\dots,g_{1,q_1}(X)),
\dots,
h_\ell(g_{\ell,1}(X),\dots,g_{\ell,q_\ell}(X))
\right).
\]
Since the $j$-th original program induces $\sigma_j$, this equals
\[
H(\sigma_{1,\clients}(X),\dots,\sigma_{\ell,\clients}(X))
=
\sigma_{\clients}(X).
\]
Therefore $\sigma$ is induced by a one-round typed program. The final claim follows from Theorem~\ref{thm:one-round-shared-state-factorization}.
\end{proof}

The corollary lets the language express composite statistics without introducing new shared-state primitives. For a positive total count, a mean follows from a sum and count by shared-only division, and a variance follows from a sum, sum of squares, and count by shared-only algebra. Total decoder maps also specify their values outside that domain, as in the mean example of Section~\ref{sec:typed-language}.

\section{Scope of the base iterative class}
\label{app:round-wise-shared-memory}

The base iterative class carries information across rounds only through shared tensors. It covers algorithms whose server-side state is updated from client-local tensor computations and shared aggregations.

\begin{proposition}[Round-wise shared memory]
\label{prop:round-wise-shared-memory}
Let an iterative typed program be given as in Definition~\ref{def:iterative-typed-program}. For every federated input $X$ and every round $t \in \{0,\dots,T\}$, the iterate $\theta_t$ is a shared tensor of type $\Sh(\boldsymbol{\tau}_t)$.

Moreover, let $e(x,\theta)$ be any client-local federated expression appearing in round $t$ of the program, with typing judgment $x:\Fed_r(\mathbf d),\theta:\Sh(\boldsymbol{\tau}_t) \vdash e(x,\theta):\Fed_{r'}(\mathbf u)$. For each fixed $\theta \in \RR^{\boldsymbol{\tau}_t}$ and each $n \in \mathbb N_0$, there exists a deterministic map $\widehat e_n(\cdot;\theta)$ such that, for every federated input $X=\{X^{(c)}\}_{c\in\clients}$ and every client $c \in \clients$ with local record count $n_c$, one has $e(X,\theta)^{(c)}=\widehat e_{n_c}(X^{(c)};\theta)$.
\end{proposition}

\begin{proof}[Proof of Proposition~\ref{prop:round-wise-shared-memory}]
We first prove that $\theta_t$ is shared for every $t \in \{0,\dots,T\}$. By Definition~\ref{def:iterative-typed-program}, the initial state satisfies $\theta_0 \in \RR^{\boldsymbol{\tau}_0}$ and has type $\Sh(\boldsymbol{\tau}_0)$. Suppose that $\theta_t$ is shared of type $\Sh(\boldsymbol{\tau}_t)$. For each $i \in [q_t]$, the expression $g_{t,i}(X,\theta_t)$ is a shared-state expression and therefore has shared type $\Sh(\mathbf s_{t,i})$. The decoder has typing judgment
\[
y_1:\Sh(\mathbf s_{t,1}),\dots,y_{q_t}:\Sh(\mathbf s_{t,q_t}),\theta:\Sh(\boldsymbol{\tau}_t)
\vdash
h_t(y_1,\dots,y_{q_t},\theta):\Sh(\boldsymbol{\tau}_{t+1}).
\]
Thus $\theta_{t+1}=h_t(g_{t,1}(X,\theta_t),\dots,g_{t,q_t}(X,\theta_t),\theta_t)$ is shared of type $\Sh(\boldsymbol{\tau}_{t+1})$. The claim follows by induction on $t$.

We now prove the local-realization statement. Fix a round $t$, fix $\theta \in \RR^{\boldsymbol{\tau}_t}$, and let $e(x,\theta)$ be a client-local federated expression with typing judgment
\[
x:\Fed_r(\mathbf d),\theta:\Sh(\boldsymbol{\tau}_t)
\vdash
e(x,\theta):\Fed_{r'}(\mathbf u).
\]
We prove by structural induction on $e$ that, for each $n \in \mathbb N_0$, there exists a deterministic map $\widehat e_n(\cdot;\theta)$ with the stated property.

If $e(x,\theta)=x$, take $\widehat e_n(\cdot;\theta)$ to be the identity map on $\mathcal X_n(r,\mathbf d)$. Then $e(X,\theta)^{(c)}=X^{(c)}=\widehat e_{n_c}(X^{(c)};\theta)$.

For the induction step, suppose that $e$ is obtained by applying a primitive symbol $F$ to subexpressions $e_1,\dots,e_m$ and shared arguments formed from $\theta$ and any fixed shared constants in the program. Since $e$ is client-local and has federated output, the final primitive application has federated output whenever it receives a federated input. By the induction hypothesis, every federated subexpression $e_j$ has deterministic local realizations $\widehat e_{j,n}(\cdot;\theta)$. Proposition~\ref{prop:client-locality} gives a deterministic map for the local output of the final primitive application as a function of the local values of its federated inputs and its shared inputs. Composing this deterministic map with the local realizations $\widehat e_{j,n}(\cdot;\theta)$ gives a deterministic map $\widehat e_n(\cdot;\theta)$ such that $e(X,\theta)^{(c)}=\widehat e_{n_c}(X^{(c)};\theta)$ for every federated input $X$ and every client $c \in \clients$.

This completes the induction and proves the proposition.
\end{proof}

Proposition~\ref{prop:round-wise-shared-memory} gives a boundary of the present language. A client may compute client-local tensor expressions allowed by the typing rules inside a round, but the only information that persists to later rounds is the shared iterate $\theta_t$. Algorithms with persistent private client variables require an extension of the type system, for example a separate local-state sort.

\section{Transient intra-round state}
\label{app:transient-intra-round-state}

This appendix extends the program interface. The tensor signature of Proposition~\ref{prop:conservative-signature-extensions} remains unchanged. A local encoder may retain temporary private summaries and model iterates within a round, then send a tensor message to a designated merge. These summaries and iterates are discarded before the next round. This permits local optimization without adding persistent private state to the language.

For each round $t$, let $M_t=\RR^{\mathbf s_{t,1}}\times\ldots\times\RR^{\mathbf s_{t,q_t}},~q_t\geq1$, where the shapes $\mathbf s_{t,i}$ do not depend on the numbers of clients or samples. Equip $M_t$ with a commutative monoid $(M_t,\oplus_t,0_t)$. The interface consists of a deterministic encoder family $E_{t,n}:\mathcal X_n(r,\mathbf d)\times\RR^{\boldsymbol\tau_t}\longrightarrow M_t,~n\in\mathbb N_0$, and a shared-only decoder $H_t:M_t\times\RR^{\boldsymbol\tau_t}\longrightarrow\RR^{\boldsymbol\tau_{t+1}}$. The family uses one local computation rule for all clients: a client with $n$ samples applies $E_{t,n}$ to its own data and the current shared state. It has no access to another client's data, temporary state, or message. Only the result of the designated merge is made shared. The interface has no operation that releases another intermediate local value. All private state used by $E_{t,n}$ is initialized from its inputs and discarded after the round.

Let $\mathcal C_t\subseteq\clients$ be a finite nonempty participating set. For the round conditioned on this set, define
\begin{equation}
\label{eq:transient-round-update}
\theta_{t+1}
=H_t\!\left(
\bigoplus_{c\in\mathcal C_t}E_{t,n_c}(X^{(c)};\theta_t),\theta_t
\right).
\end{equation}
Sampling $\mathcal C_t$, if used, is external to this deterministic update. This equation describes the update on the participating data, and it does not identify a sampled gradient with the gradient of the full federation.

\begin{proposition}[Transient-encoder factorization]
\label{prop:transient-encoder-factorization}
For the round update in Equation~\ref{eq:transient-round-update}, assume that $M_t$ has count-independent tensor shapes and a commutative monoid merge, that the encoders use one client-local computation rule with private state initialized from their inputs and discarded after each round, and that only the designated merge produces shared state. With a shared-only decoder $H_t$, the round factors through $M_t$ with local encoders $\phi_{t,n}=E_{t,n}$ and decoder $\psi_t=H_t$. Each client's temporary state depends only on its own data and the current shared state, and shared-state formation occurs only through the designated merge. No private state persists across rounds.
\end{proposition}

\begin{proof}
Substituting $\phi_{t,n}=E_{t,n}$ and $\psi_t=H_t$ into Equation~\ref{eq:transient-round-update} gives the factorization. The encoder family provides a map for every local sample count, independent of client identity. Its inputs restrict the data available to each client, and its codomain has dimension independent of the client and sample counts. The interface conditions restrict shared-state formation to the merge and discard temporary private state after the round.
\end{proof}

The exposure condition in this interface concerns which local values become shared. Proposition~\ref{prop:exposure-discipline} retains its primitive classification for the syntax; the present interface adds a private-message merge path. Proposition~\ref{prop:virtual-global-consistency} also does not extend to this interface in general, as the weighted FedAvg and FedProx example below shows.

\begin{example}[Weighted FedAvg and FedProx]
Let the shared parameter lie in $\RR^p$, and let $K\geq1$ be the number of local steps in a round. For $n>0$, assume a real-valued loss $\ell(z,w)$ on the sample space and $\RR^p$, differentiable in $w$, and define the local mean objective
\[
F_n(Z,w)=\frac1n\sum_{a=1}^n\ell(Z_a,w).
\]
The notation $Z_a$ denotes the $a$-th sample along axis $r$. Initialize $w_{t,0}(Z;\theta)=\theta$. For example, full-batch local gradient descent uses
\[
w_{t,k+1}
=w_{t,k}-\eta_{t,k}\nabla_wF_n(Z,w_{t,k}),
~ k=0,\ldots,K-1.
\]
FedProx uses the same construction with objective
\[
F_n(Z,w)+\frac\mu2\|w-\theta\|^2,
~ \mu\geq0.
\]
The proximal center is the current shared parameter, so this construction requires no private state from an earlier round.

For either local objective, use $M_t=\RR^p\times\RR$ with ordinary addition and identity $(\mathbf0,0)$. Define
\[
E_{t,n}(Z;\theta)=
\begin{cases}
(n\,w_{t,K}(Z;\theta),n),&n>0,\\
(\mathbf0,0),&n=0,
\end{cases}
\]
and
\[
H_t((A,N),\theta)=
\begin{cases}
A/N,&N>0,\\
\theta,&N\leq0.
\end{cases}
\]
An empty client sends the merge identity without evaluating a local mean objective. On messages produced by the encoders, $N=\sum_{c\in\mathcal C_t}n_c\geq0$. For $N>0$, the decoder gives the sample-weighted average of the local models. For $N=0$, it retains the current shared parameter. The $N\leq0$ branch extends the decoder to all of $M_t\times\RR^p$.

All $K$ local steps occur inside one encoder before one merge. Thus, one logical communication round comprises distribution of the shared model, local computation, and message aggregation; $T$ counts these rounds rather than individual local steps. Physical network scheduling is separate from this semantic interface. Stochastic local steps can instead be represented by a random encoder, using the same computation rule with an additional local random seed. Conditional on those seeds, Equation~\ref{eq:transient-round-update} gives the corresponding deterministic update. This does not add a convergence or privacy result.
\end{example}

\begin{example}[Dependence on the data partition]
Consider scalar losses $\ell_h(w)=hw^2/2$ for samples $h=1,3$, with initial parameter $\theta=1$, step size $\eta=1/4$, and $K=2$. Two clients with one sample each produce
\[
w_{1,2}=(1-1/4)^2=9/16,
~
w_{3,2}=(1-3/4)^2=1/16.
\]
Their sample-weighted average is $5/16$. A single client containing both samples has mean objective $F_2(w)=w^2$, so two local steps give $(1-2/4)^2=1/4$. Both partitions have the same virtual global input, but their updates differ. This prevents a general virtual global-consistency conclusion for the transient interface.
\end{example}

Methods that retain client-specific control variables between rounds, such as SCAFFOLD or FedDyn, require persistent local state and participation-dependent state updates. They are outside both the base iterative class and the transient interface of this appendix.

\section{Proofs for learning updates}
\label{app:learning-update-proofs}

\begin{proof}[Proof of Theorem~\ref{thm:gradient-expressibility}]
Let $(\ell,\gamma)$ be a representably differentiable loss. Fix a finite ordered federation $\clients$ and a federated input $X=\{X^{(c)}\}_{c\in\clients}$ of type $\Fed_r(\mathbf d)$, with local record counts $\{n_c\}_{c\in\clients}$.

By Definition~\ref{def:representably-differentiable-loss}, for each client $c \in \clients$, each record index $a \in [n_c]$, and each $\theta \in \RR^{\boldsymbol{\tau}}$, the map $\theta \mapsto \widehat \ell_{n_c}(X^{(c)};\theta)[a]$ is differentiable and satisfies
\[
\nabla_\theta \widehat \ell_{n_c}(X^{(c)};\theta)[a]
=
\widehat \gamma_{n_c}(X^{(c)};\theta)[a,\cdot].
\]
The empirical objective $L_{\clients}(X,\theta)$ is a finite sum of these differentiable scalar functions. Therefore it is differentiable, and the gradient of a finite sum is the sum of the gradients. Hence
\[
\nabla_\theta L_{\clients}(X,\theta)
=
\sum_{c\in\clients}\sum_{a=1}^{n_c}
\nabla_\theta \widehat \ell_{n_c}(X^{(c)};\theta)[a]
=
\sum_{c\in\clients}\sum_{a=1}^{n_c}
\widehat \gamma_{n_c}(X^{(c)};\theta)[a,\cdot]
=
G_{\clients}(X,\theta).
\]

It remains to prove the expressibility claim. By the typing assumption in Definition~\ref{def:representably-differentiable-loss},
\[
x:\Fed_r(\mathbf d),\theta:\Sh(\boldsymbol{\tau})
\vdash
\gamma(x,\theta):\Fed_1(\boldsymbol{\tau}).
\]
Since $\operatorname{Sum} \in \mathcal A$, the aggregation typing rule gives
\[
x:\Fed_r(\mathbf d),\theta:\Sh(\boldsymbol{\tau})
\vdash
\operatorname{Sum}_1(\gamma(x,\theta)):\Sh(\boldsymbol{\tau}).
\]
For any $X$ and $\theta$, Definition~\ref{def:aggregation-semantics} gives
\[
\operatorname{Sum}_1(\gamma(X,\theta))
=
\operatorname{Sum}_1(\vglob{\gamma(X,\theta)}).
\]
By the local realization of the client-local expression $\gamma$, the local tensor $\gamma(X,\theta)^{(c)}$ is equal to $\widehat \gamma_{n_c}(X^{(c)};\theta)$ for every client $c$. Therefore
\[
\operatorname{Sum}_1(\gamma(X,\theta))
=
\sum_{c\in\clients}\sum_{a=1}^{n_c}
\widehat \gamma_{n_c}(X^{(c)};\theta)[a,\cdot]
=
G_{\clients}(X,\theta).
\]
The merge operation for this record-axis summation is ordinary addition in $\RR^{\boldsymbol{\tau}}$. Thus the gradient computation is represented by the shared-state expression $g(x,\theta)=\operatorname{Sum}_1(\gamma(x,\theta))$ and has the claimed round-wise shared-state form.
\end{proof}

\begin{proof}[Proof of Corollary~\ref{cor:server-side-first-order-updates}]
We construct an iterative typed program inducing the stated update. Fix a round $t \in \{0,\dots,T-1\}$. The current shared state has type $\Sh(\boldsymbol{\zeta}_t)$. The shared-only expression $P_t$ maps it to a parameter of type $\Sh(\boldsymbol{\tau})$.

By Definition~\ref{def:representably-differentiable-loss}, the per-record gradient expression has typing judgment
\[
x:\Fed_r(\mathbf d),\theta:\Sh(\boldsymbol{\tau})
\vdash
\gamma(x,\theta):\Fed_1(\boldsymbol{\tau}).
\]
Substituting the shared-only expression $P_t(s)$ for $\theta$ gives a client-local federated expression
\[
x:\Fed_r(\mathbf d),s:\Sh(\boldsymbol{\zeta}_t)
\vdash
\gamma(x,P_t(s)):\Fed_1(\boldsymbol{\tau}).
\]
Define the round-$t$ shared-state expression
\[
g_t(x,s):=\operatorname{Sum}_1(\gamma(x,P_t(s))).
\]
By the aggregation typing rule, $g_t(x,s)$ has type $\Sh(\boldsymbol{\tau})$. It has the aggregation form allowed in Definition~\ref{def:iterative-typed-program}, with ordinary addition as the merge operation.

Let the round-$t$ decoder be
\[
h_t(y,s):=H_t(y,s).
\]
This is a shared-only expression of type $\Sh(\boldsymbol{\zeta}_{t+1})$ by assumption. Therefore the expressions $g_t$ and $h_t$ define a valid round of an iterative typed program. Repeating this construction for $t=0,\dots,T-1$ gives an iterative typed program with initial state $s_0$.

It remains to identify the induced update. Let $\clients$ be a finite ordered federation, let $X=\{X^{(c)}\}_{c\in\clients}$ be a federated input of type $\Fed_r(\mathbf d)$, and let $s_t \in \RR^{\boldsymbol{\zeta}_t}$ be the current shared optimizer state. By Theorem~\ref{thm:gradient-expressibility}, the shared-state expression $g_t(X,s_t)=\operatorname{Sum}_1(\gamma(X,P_t(s_t)))$ evaluates to $G_{\clients}(X,P_t(s_t))$. Hence the round update induced by the program is
\[
s_{t+1}
=
h_t(g_t(X,s_t),s_t)
=
H_t(G_{\clients}(X,P_t(s_t)),s_t).
\]
This is the update in the statement. Therefore $X \mapsto s_T$ is induced by an iterative typed program. The final claim follows from Theorem~\ref{thm:iterative-shared-state-factorization}.
\end{proof}

\begin{proof}[Proof of Corollary~\ref{cor:curvature-block-updates}]
We construct an iterative typed program inducing the stated update. Fix a round $t \in \{0,\dots,T-1\}$. The current shared state has type $\Sh(\boldsymbol{\zeta}_t)$, and the shared-only expression $P_t$ maps it to a parameter of type $\Sh((p))$.

By Definition~\ref{def:representably-differentiable-loss}, the per-record gradient expression has typing judgment
\[
x:\Fed_r(\mathbf d),\theta:\Sh((p))
\vdash
\gamma(x,\theta):\Fed_1((p)).
\]
Substituting the shared-only expression $P_t(s)$ for $\theta$ gives
\[
x:\Fed_r(\mathbf d),s:\Sh(\boldsymbol{\zeta}_t)
\vdash
\gamma(x,P_t(s)):\Fed_1((p)).
\]
Define the gradient shared-state expression
\[
g_t(x,s):=\operatorname{Sum}_1(\gamma(x,P_t(s))).
\]
By the aggregation typing rule, $g_t(x,s)$ has type $\Sh((p))$ and has the aggregation form allowed in Definition~\ref{def:iterative-typed-program}.

Similarly, the curvature expression $B$ has typing judgment
\[
x:\Fed_r(\mathbf d),\theta:\Sh((p))
\vdash
B(x,\theta):\Fed_1((p,p)).
\]
Substituting $P_t(s)$ for $\theta$ gives
\[
x:\Fed_r(\mathbf d),s:\Sh(\boldsymbol{\zeta}_t)
\vdash
B(x,P_t(s)):\Fed_1((p,p)).
\]
Define the curvature shared-state expression
\[
C_t(x,s):=\operatorname{Sum}_1(B(x,P_t(s))).
\]
By the aggregation typing rule, $C_t(x,s)$ has type $\Sh((p,p))$ and has the aggregation form allowed in Definition~\ref{def:iterative-typed-program}.

Let the round-$t$ decoder be $h_t(y,C,s):=H_t(y,C,s)$. This is a shared-only expression of type $\Sh(\boldsymbol{\zeta}_{t+1})$ by assumption. The shared-linear-algebra primitives used inside $H_t$ act only on shared tensors, so they are part of the shared-only decoder and do not affect the client-local expressions. Thus $g_t$, $C_t$, and $h_t$ define a valid round of an iterative typed program in the shared-linear-algebra extension of the language. Repeating this construction for $t=0,\dots,T-1$ gives an iterative typed program with initial state $s_0$.

It remains to identify the induced update. Let $\clients$ be a finite ordered federation, let $X=\{X^{(c)}\}_{c\in\clients}$ be a federated input of type $\Fed_r(\mathbf d)$, and let $s_t \in \RR^{\boldsymbol{\zeta}_t}$ be the current shared optimizer state. By Theorem~\ref{thm:gradient-expressibility}, the expression $g_t(X,s_t)$ evaluates to $G_{\clients}(X,P_t(s_t))$. By Definition~\ref{def:aggregation-semantics} and the local realization of $B$, the expression $C_t(X,s_t)$ evaluates to
\[
\sum_{c\in\clients}\sum_{a=1}^{n_c}
\widehat B_{n_c}(X^{(c)};P_t(s_t))[a,\cdot,\cdot]
=
C_{\clients}(X,P_t(s_t)).
\]
Therefore the round update induced by the program is
\[
s_{t+1}
=
h_t(g_t(X,s_t),C_t(X,s_t),s_t)
=
H_t(G_{\clients}(X,P_t(s_t)),C_{\clients}(X,P_t(s_t)),s_t).
\]
This is the update in the statement. Therefore $X \mapsto s_T$ is induced by an iterative typed program. The final claim follows from Theorem~\ref{thm:iterative-shared-state-factorization}, whose proof applies because the added linear-algebra primitives are shared-only.
\end{proof}

\section{Learning examples}
\label{app:learning-examples}

This appendix presents two learning examples for Section~\ref{sec:differentiable-learning}. We use client-local feature and response projections and the following operations under Proposition~\ref{prop:conservative-signature-extensions}. For $U:\Fed_1(\varnothing)$ and $V:\Fed_1((p))$ with the same local sample counts and sample order, define
\[
\operatorname{ScaleRec}(U,V)^{(c)}[a,j]
  :=U^{(c)}[a]V^{(c)}[a,j].
\]
Its output has type $\Fed_1((p))$. Its ordinary interpretation reshapes the scalar-sample vector from $(n,)$ to $(n,1)$ and multiplies it elementwise with the feature matrix of shape $(n,p)$. For $V:\Fed_1((p))$ and $\theta:\Sh((p))$, define
\[
\operatorname{Pred}(V,\theta)^{(c)}[a]
  :=\sum_{j=1}^pV^{(c)}[a,j]\theta_j,
\]
of type $\Fed_1(\varnothing)$. Its ordinary interpretation is the row-wise dot product. It can also be represented by a federated-shared matrix product after reshaping $\theta$ to $(p,1)$ and removing the output's non-sample singleton axis. We also use the sample-wise outer product $\operatorname{Outer}(v,v)$, defined in this appendix. Each operation processes samples independently using the same rule at every client. For aligned sample blocks, concatenating the client outputs gives the same result as applying the ordinary tensor operation to the concatenated inputs (Appendix~\ref{app:conservative-extensions}).

\begin{example}[Logistic regression with server-side first-order optimization]
\label{ex:logistic-first-order}
Consider binary logistic regression with covariates $x_a \in \RR^p$, responses $y_a \in \{0,1\}$, and parameter $\theta \in \RR^p$. Let $\rho(t)=(1+\exp(-t))^{-1}$. The per-record negative log-likelihood and gradient are $\ell_a(\theta) = \log(1+\exp(x_a^\top\theta)) - y_a x_a^\top\theta$ and $\gamma_a(\theta) = (\rho(x_a^\top\theta)-y_a)x_a$. After adding the scalar maps used in the loss and gradient, coordinate projections, and the aligned operations $\operatorname{ScaleRec}$ and $\operatorname{Pred}$, these are client-local tensor expressions. Hence $(\ell,\gamma)$ is a representably differentiable loss. By Theorem~\ref{thm:gradient-expressibility}, the global gradient is the record-axis sum of the federated gradient tensor. By Corollary~\ref{cor:server-side-first-order-updates}, gradient descent, server-side momentum, and server-side Adam updates for this model are induced by iterative typed programs.
\end{example}

Let the federated record tensor contain, for each record, a covariate vector in $\RR^p$ and a binary response. We write $X_{\mathrm{feat}}$ and $Y$ for the client-local projections onto covariates and responses. These projections are conservative client-local primitives in the sense of Proposition~\ref{prop:conservative-signature-extensions}.

Let $\theta:\Sh((p))$ and define
\[
R(X,\theta):=\operatorname{Pred}(X_{\mathrm{feat}},\theta),
~
P(X,\theta):=\rho(R(X,\theta)).
\]
The gradient tensor is
\[
\gamma(X,\theta)
:=
\operatorname{ScaleRec}(P(X,\theta)-Y,X_{\mathrm{feat}}).
\]
The residual and features are aligned projections of the same input, so the expression has type $\Fed_1((p))$. Thus $\operatorname{Sum}_1(\gamma(X,\theta)):\Sh((p))$ evaluates to the global logistic-regression gradient. Server-side gradient descent, momentum, and Adam follow from Corollary~\ref{cor:server-side-first-order-updates} with the corresponding shared-only update expression.

\begin{example}[Gaussian linear regression with a damped Newton update]
\label{ex:gaussian-linear-newton}
Consider the Gaussian linear model $y_a \mid x_a,\theta \sim \mathcal N(x_a^\top\theta,\sigma^2)$ with covariates $x_a \in \RR^p$, responses $y_a \in \RR$, parameter $\theta \in \RR^p$, and known variance $\sigma^2>0$. Up to an additive constant, the per-record negative log-likelihood is $\ell_a(\theta)=(y_a-x_a^\top\theta)^2/(2\sigma^2)$. The per-record gradient is $\gamma_a(\theta) = (x_a^\top\theta-y_a)x_a/\sigma^2$, and the per-record curvature block is $B_a = x_a x_a^\top/\sigma^2$. Therefore
\[
G_{\clients}(X,\theta)=
\sum_{c\in\clients}\sum_{a=1}^{n_c}
\frac{1}{\sigma^2}(x_{c,a}^\top\theta-y_{c,a})x_{c,a}
\]
and
\[
C_{\clients}(X,\theta)=
\sum_{c\in\clients}\sum_{a=1}^{n_c}
\frac{1}{\sigma^2}x_{c,a}x_{c,a}^\top.
\]
Both are record-axis sums of client-local tensor expressions, so the damped Newton step
\[
\theta^+ = \theta-\eta\,\operatorname{Solve}(C_{\clients}(X,\theta)+\lambda I,G_{\clients}(X,\theta))
\]
is an instance of Corollary~\ref{cor:curvature-block-updates}, when the shared linear system is defined.
\end{example}

Let the federated record tensor contain, for each record, a covariate vector in $\RR^p$ and a real-valued response. We again write $X_{\mathrm{feat}}$ and $Y$ for the client-local projections onto covariates and responses. Consider the Gaussian linear model $y_a \mid x_a,\theta \sim \mathcal N(x_a^\top\theta,\sigma^2)$, with known $\sigma^2>0$.

Let $\theta:\Sh((p))$ and define the residual
\[
R(X,\theta):=\operatorname{Pred}(X_{\mathrm{feat}},\theta)-Y.
\]
The gradient tensor is
\[
\gamma(X,\theta):=\frac{1}{\sigma^2}
\operatorname{ScaleRec}(R(X,\theta),X_{\mathrm{feat}}),
\]
which has type $\Fed_1((p))$. Let $\operatorname{Outer}(v,v) \in \RR^{p\times p}$ denote the matrix with entries $\operatorname{Outer}(v,v)_{jk}=v_jv_k$. The per-record curvature tensor is
\[
B(X,\theta)[a,:,:] := \frac{1}{\sigma^2}\operatorname{Outer}(X_{\mathrm{feat}}[a,:],X_{\mathrm{feat}}[a,:]),
\]
which has type $\Fed_1((p,p))$. Thus $\operatorname{Sum}_1(\gamma(X,\theta))$ and $\operatorname{Sum}_1(B(X,\theta))$ are shared tensors of types $\Sh((p))$ and $\Sh((p,p))$. A damped Newton decoder is the shared-only expression
\[
H(g,C,\theta)
=
\theta-\eta\,\operatorname{Solve}(C+\lambda I,g),
\]
defined on the domain where $C+\lambda I$ is non-singular. This is the form covered by Corollary~\ref{cor:curvature-block-updates}.

\section{Conservative extensions of the primitive signature}
\label{app:conservative-extensions}

The base signature is intentionally small. It contains only the primitive families needed to state the core separation between client-local computation and shared-state formation. This keeps the semantics of Section~\ref{sec:typed-language} and the factorization results of Section~\ref{sec:shared-state-factorizations} independent of model-specific operations. Applications may require coordinate projections, reshaping, row-wise maps, convolutions, or model-specific differentiable maps. Their use as conservative extensions requires the locality and ordinary-semantics conditions of Proposition~\ref{prop:conservative-signature-extensions}. A convolution using samples across client boundaries fails these conditions. They may also need shared-only primitives, such as matrix multiplication or linear solves. The factorization results do not depend on the particular list of such primitives. They depend on whether the added primitives preserve the same separation between client-local computation and shared-state formation.

\begin{proposition}[Conservative signature extensions]
\label{prop:conservative-signature-extensions}
Let $\Sigma^+$ be an extension of the base signature $\Sigma_0$. Assume that every primitive symbol in $\Sigma^+\setminus\Sigma_0$ is of one of the following two kinds.

First, it may be shared-only. In this case it is defined only on shared inputs and returns a shared tensor.

Second, it may be client-local. In this case, whenever it receives at least one federated input, it returns a federated tensor, and for every typed primitive application with federated output $Y$ there is, for each client $c \in \clients$, a deterministic map $\Phi_c$ such that $Y^{(c)}=\Phi_c(X_1^{(c)},\dots,X_p^{(c)},S_1,\dots,S_q)$, where $X_1,\dots,X_p$ are the federated inputs and $S_1,\dots,S_q$ are the shared inputs of that primitive application. The maps $\Phi_c$ may depend on the primitive symbol, the operand positions, and the local input shapes at client $c$, but not on the identity of $c$ or on any tensor stored at another client.

Each added symbol $F$ has an ordinary interpretation $F^{\operatorname{ord}}$ defined on the ordinary inputs corresponding to its permitted typed applications. An all-shared application of a client-local symbol, if permitted, returns a shared tensor. For every added primitive application $Y=F(A_1,\dots,A_m)$, set
\[
\widetilde A_i=
\begin{cases}
\vglob{A_i}, & A_i\text{ is federated},\\
A_i, & A_i\text{ is shared}.
\end{cases}
\]
Require ordinary agreement on the application's domain:
\[
\begin{cases}
\vglob{Y}=F^{\operatorname{ord}}(\widetilde A_1,\dots,\widetilde A_m),
  & Y\text{ is federated},\\
Y=F^{\operatorname{ord}}(A_1,\dots,A_m),
  & Y\text{ is shared}.
\end{cases}
\]
Any structural alignment conditions and domains of partial operations are part of the permitted application.

Then Propositions~\ref{prop:client-locality}, \ref{prop:exposure-discipline}, and~\ref{prop:virtual-global-consistency} remain valid for typed primitive applications in $\Sigma^+$. Moreover, Theorems~\ref{thm:one-round-shared-state-factorization} and~\ref{thm:iterative-shared-state-factorization} remain valid for one-round and iterative typed programs over $\Sigma^+$, provided shared-state formation occurs only through the record-axis elimination primitives allowed in Definitions~\ref{def:one-round-typed-program} and~\ref{def:iterative-typed-program}. For this theorem-preservation conclusion, the programs' local encoders and shared decoders are defined on the full input and state domains in those theorems and propositions, including all shared parameter values in the iterative case. Operations with restricted domains use total extensions.
\end{proposition}

Packing and component extraction from Subsection~\ref{subsec:iterative-typed-programs} are total shared-only primitives. Taking their ordinary interpretations to be the same packing and extraction operations satisfies the premises of Proposition~\ref{prop:conservative-signature-extensions}.

\begin{proof}[Proof of Proposition~\ref{prop:conservative-signature-extensions}]
We first prove the three primitive semantic propositions for typed primitive applications in $\Sigma^+$.

For symbols in the base signature $\Sigma_0$, the claims are Propositions~\ref{prop:client-locality}, \ref{prop:exposure-discipline}, and~\ref{prop:virtual-global-consistency}.

Let $F\in\Sigma^+\setminus\Sigma_0$. An all-shared application returns a shared tensor and agrees with $F^{\operatorname{ord}}$ by the ordinary-agreement assumption. An application with federated inputs returns a federated tensor and has the assumed uniform local realization. Its virtual global tensor equals $F^{\operatorname{ord}}$ applied to the ordinary inputs by the same assumption. Thus, the added symbols satisfy uniform locality and virtual-global consistency and create no new federated-to-shared boundary.

Therefore Propositions~\ref{prop:client-locality}, \ref{prop:exposure-discipline}, and~\ref{prop:virtual-global-consistency} hold for every typed primitive application in $\Sigma^+$.

It remains to prove that the shared-state factorization theorems remain valid. The proof of Theorem~\ref{thm:one-round-shared-state-factorization} uses the primitive semantics only through the client-local realization of client-local federated expressions and through the specified merge identities of the shared-state expressions. For programs over $\Sigma^+$, the same client-local realization follows by structural induction on the expression. The induction step uses the extended version of Proposition~\ref{prop:client-locality}, proved above. The rest of the proof is unchanged, because shared-state formation is assumed to occur only through the record-axis elimination primitives allowed in Definition~\ref{def:one-round-typed-program}.

The same argument applies to Theorem~\ref{thm:iterative-shared-state-factorization}. Its proof uses primitive semantics only through client-local realization of round-wise federated expressions and through the prescribed merge identities of the shared-state expressions. The extended client-locality proposition gives the same local realization by structural induction, with the current shared iterate treated as a shared input. Since the allowed shared-state formation mechanisms are unchanged, the encode, merge, and decode construction is identical.

The full-domain premise makes the constructed encoder and decoder maps total on the domains required by the two theorems. Thus, both factorization theorems remain valid with the same valid shared-output boundaries.
\end{proof}

This proposition is used as a conservative extension principle. It allows richer client-local model computations and richer shared-only decoders, while preserving the shared-state factorization argument.

Shared linear solves illustrate the treatment of restricted domains. For $A:\Sh((p,p))$ and $b:\Sh((p))$, define the shared-only extension
\[
\operatorname{Solve}_0(A,b):=
\begin{cases}
\operatorname{Solve}(A,b), & A\text{ is nonsingular},\\
\mathbf 0, & A\text{ is singular}.
\end{cases}
\]
It has output type $\Sh((p))$, and its ordinary interpretation uses the same two branches, so it satisfies the proposition's setup. It is measurable because the nonsingular matrices form an open set and inversion is continuous on that set. Using this extension in a decoder whose only restricted operation is a linear solve gives a decoder defined on the full product state space. The update agrees with the original solve on input and state sequences for which every required matrix is nonsingular. This convention adds no new shared-output boundary.

The operations $\operatorname{ScaleRec}$ and $\operatorname{Pred}$ are defined before use in Appendix~\ref{app:learning-examples}. They satisfy the proposition's premises. For aligned inputs with local sample counts $n_c$, write $b_c=\sum_{d<c}n_d$ for the offset in their virtual global tensors. Then
\[
\begin{aligned}
\vglob{\operatorname{ScaleRec}(U,V)}[b_c+a,j]
 &= U^{(c)}[a]V^{(c)}[a,j]\\
 &= \vglob{U}[b_c+a]\,\vglob{V}[b_c+a,j],\\
\vglob{\operatorname{Pred}(V,\theta)}[b_c+a]
 &= \sum_{j=1}^p\vglob{V}[b_c+a,j]\theta_j.
\end{aligned}
\]
These are the ordinary column-scaled product and row-wise dot product. Their local maps depend on the local shapes and shared parameter, and use no other client's input. When $n_c=0$, the corresponding index set is empty and the same identities hold. Coordinate projections and the sample-wise outer product satisfy the ordinary-agreement assumption by the same entrywise argument. None of these operations returns a shared tensor from a federated input.

\section{Implementation and empirical evaluation}
\label{app:implementation}

FeTA (Federated Tensor Algebra) provides a Python implementation for the two sorts of Section~\ref{sec:typed-language}. FeTA represents federated and shared values by the wrapper classes \texttt{FederatedTensor} and \texttt{SharedTensor}. A user writes tensor operations, and elimination of the tracked federated axis identifies an aggregation boundary. Theorem~\ref{thm:one-round-shared-state-factorization} and its extraction construction describe the corresponding mathematical encoder, merge, and decoder.

A controller/executor interface connects the tensor program to backend orchestration, including an NVIDIA FLARE adapter. Runtime wrappers check placement, local shapes, alignment of federated operands, and operation effects when tensor values are available. Two federated operands in an elementwise operation must have matching local shapes. A shared operand must satisfy the broadcast condition of Definition~\ref{def:record-agnostic-broadcast-compatibility}. A federated-federated matrix product may expose shared state only through the permitted contraction of the two tracked axes. Rejection of an invalid operation precedes the collective associated with that operation; earlier operations in the program may already have communicated.

An alternative aggregation backend must satisfy the required tensor interface and its own numerical and security assumptions. A secure merge requires a protocol and threat model; a differentially private mechanism also requires sensitivity control and randomization. The privacy conditions in Appendix~\ref{app:privacy-integration} therefore remain separate from the choice of execution backend.

\subsection{Covariance example and communication dependencies}

\begin{lstlisting}[caption={FeTA covariance implementation for aligned one-dimensional federated arrays with total sample count $N>0$, using population means.},
label={lst:fed-cov}]
def cov(fed_x: FederatedTensor, fed_y: FederatedTensor):
    """Population covariance between aligned arrays."""
    mean_x = fed_x.mean()          # aggregation over samples
    mean_y = fed_y.mean()          # aggregation over samples
    return ((fed_x - mean_x) * (fed_y - mean_y)).mean()
\end{lstlisting}

Let $x^{(c)},y^{(c)}\in\mathbb R^{n_c}$ contain paired samples, and let $N=\sum_c n_c>0$. The semantics of each \texttt{mean} is the merge of local sums and counts, followed by division by the total count. Thus, clients with unequal $n_c$ receive their sample weights, and an empty client contributes $(0,0)$. With $\bar x=N^{-1}\sum_{c,a}x^{(c)}_a$ and $\bar y=N^{-1}\sum_{c,a}y^{(c)}_a$, the displayed expression computes
\[
  \frac{1}{N}\sum_c\sum_{a=1}^{n_c}
  (x^{(c)}_a-\bar x)(y^{(c)}_a-\bar y),
\]
the population covariance, corresponding to \texttt{ddof=0}. This example has no value specified when $N=0$.

The source contains three reductions. The first two are independent, and the centered-product reduction depends on their shared results. There are two dependency layers if the first two reductions are batched; an eager backend may execute three collectives. Consequently, this displayed source is not a one-round program under Definition~\ref{def:one-round-typed-program}. The user code contains no client loop or message-passing routine. Operations over the tracked axis identify the aggregation boundaries.

TFF can express the same algorithm by composing local sum/count maps, federated aggregation of those summaries, broadcast of the resulting means, and a local centered-product map followed by aggregation. FeTA obtains those boundaries from tensor operations. This comparison concerns source organization and typing, and makes no claim of fewer necessary communication phases.

\subsection{Reported code-complexity results}

The evaluation compares five tasks across FeTA, Flare, Flower, Ray, Syft, and TFF: K-Means, logistic regression (LR), histogram computation (Hist), a suite of eleven summary-statistics routines (Stats), and a generalized linear mixed model (GLMM). The summary-statistics suite includes Pearson correlation and standardized mean difference. Each task uses a common centralized reference. Tables~\ref{tab:feta-lloc} and~\ref{tab:feta-mi} give the logical lines of code (LLOC) and maintainability index (MI) reported for these implementations, computed with Radon.

\begin{table}[t]
\centering
\begin{tabular}{lrrrrrr}
\toprule
Task & FeTA & Flare & Flower & Ray & Syft & TFF \\
\midrule
K-Means & 48 & 147 & 149 & 65 & 56 & 55 \\
LR & 38 & 124 & 135 & 65 & 59 & 54 \\
Hist & 35 & 119 & 106 & 40 & 40 & 36 \\
Stats & 42 & 138 & 127 & 56 & 90 & 52 \\
GLMM & 134 & 272 & 288 & 155 & 139 & 148 \\
\bottomrule
\end{tabular}
\caption{Reported LLOC for the five tasks. Lower values indicate less counted source code under the evaluation's counting convention.}
\label{tab:feta-lloc}
\end{table}

\begin{table}[t]
\centering
\begin{tabular}{lrrrrrr}
\toprule
Task & FeTA & Flare & Flower & Ray & Syft & TFF \\
\midrule
K-Means & 46.64 & 31.38 & 26.42 & 42.22 & 44.14 & 43.80 \\
LR & 48.43 & 33.08 & 28.94 & 39.87 & 41.44 & 41.76 \\
Hist & 48.96 & 32.99 & 33.62 & 46.20 & 46.33 & 47.20 \\
Stats & 44.98 & 28.14 & 28.31 & 37.05 & 29.02 & 37.76 \\
GLMM & 27.40 & 17.02 & 14.60 & 23.71 & 26.49 & 24.22 \\
\bottomrule
\end{tabular}
\caption{Reported MI for the same implementations. Higher scores indicate greater estimated code maintainability under this metric.}
\label{tab:feta-mi}
\end{table}

In this collection, FeTA has the lowest LLOC and highest MI for all five tasks. These metrics are proxies for code complexity, rather than measurements of runtime performance or developer productivity. The reported validation found agreement with centralized references on the tested configurations. This is not a proof of semantic correctness for arbitrary Python programs.

\section{Complexity of the Shared-State Factorization}
\label{app:complexity}

Consider a chosen evaluation of the one-round factorization
\[
  \sigma_{\mathcal C}(X)
  =\psi\!\left(\bigoplus_{c\in\mathcal C}
       \varphi_{n_c}(X^{(c)})\right),
  ~
  m=|\mathcal C|\geq1,
  ~
  D=\dim M=\sum_{i=1}^q\prod_j(\mathbf s_i)_j,
\]
where an empty shape contributes one coordinate. The dimension $D$ depends on the program's shared-state shapes, independently of $m$ and the local sample counts. This coordinate bound gives no bound on the cost of an arbitrary encoder, merge, or decoder.

Let $T_\varphi(n_c)$, $T_\oplus(D)$, and $T_\psi(D)$ be costs for the chosen encoder, pairwise merge, and decoder. These symbols may denote upper bounds; dependence on tensor values, primitive implementations, and program parameters is suppressed. Initializing a sequential fold with the first client message gives the work decomposition
\begin{equation}
\label{eq:factorization-work}
  W_{\mathrm{seq}}
  \leq
  \sum_{c\in\mathcal C}T_\varphi(n_c)
  +(m-1)T_\oplus(D)+T_\psi(D).
\end{equation}
The same expression is the work sum when the cost symbols denote stage costs and each merge call has the displayed cost. A fold initialized at the monoid identity instead uses $m$ merge calls. Theorem~ \ref{thm:one-round-shared-state-factorization} supplies the factorization.

\textbf{Merge storage and depth.}
Suppose a merge can be evaluated with $O(D)$ workspace. A streaming fold then uses $O(D)$ auxiliary merge storage, independently of $m$. This bound excludes encoder and decoder workspace, retained data, and simultaneously buffered messages. Associativity permits a balanced tree with $\lceil\log_2m\rceil$ merge stages. If encoders execute concurrently, merges at each stage execute concurrently, and each merge has cost at most $T_\oplus(D)$, the arithmetic critical path is at most
\[
  \max_cT_\varphi(n_c)
  +\lceil\log_2m\rceil T_\oplus(D)+T_\psi(D).
\]
Transport latency, bandwidth, scheduling, and available processors require a separate model before this expression becomes a wall-clock bound.

\textbf{Dense messages and elementwise merges.}
If every client transmits one dense, uncompressed element of $M$, the uplink contains $\Theta(mD)$ real coordinates. If the merge processes each coordinate in constant time, the $m-1$ pairwise merges require $\Theta((m-1)D)$ work, with zero merge work for $m=1$. These bounds refer to that communication scheme and merge implementation. They are not minimum communication bounds. A bit count requires a representation and precision model. The expression excludes broadcasts of shared parameters and other protocol traffic. Decoder work is the separate term $T_\psi(D)$.

\textbf{Learning examples.}
A summed gradient in $\mathbb R^p$ has $D=p$. An unknown sample count adds a scalar if it is not already shared. Momentum and Adam maintain $O(p)$ shared server state in addition to the gradient message. For the dense curvature update of Corollary~\ref{cor:curvature-block-updates}, a message containing $g\in\mathbb R^p$ and $C\in\mathbb R^{p\times p}$ has $D=p+p^2$. A conventional dense linear solve may require $O(p^3)$ arithmetic, so a small dependence on client count does not imply a small decoder cost. Diagonal, block, low-rank, or sketch representations can reduce transmitted state only when supported by the corresponding mergeable expressions. Without the required matrix structure, such representations change or approximate the update. A rank-$r$ label does not provide a mergeable $O(pr)$ representation.

\textbf{Round-wise use.}
At round $t$, let $D_t=\sum_{i=1}^{q_t}\prod_j(\mathbf s_{t,i})_j$, with current shared state treated as an encoder parameter. The same decomposition applies with $D_t$, the participating set $\mathcal C_t$, and the costs of that round. Local work may depend on local sample count and the number of local optimizer steps. Message dimension and persistent shared memory are separate quantities; neither bounds local work. Virtual-global agreement concerns the result of evaluation and does not assert equality of federated and centralized running times.

\textbf{Participation, ordering, and topology.}
For a sampled finite nonempty set $\mathcal C_t$, condition on that set and the current shared state, then apply the deterministic factorization to the participating federation. Sampling provides external randomness. This does not identify a sampled gradient with the full-data gradient or establish unbiasedness or convergence. In the sum/count and weighted-averaging examples, empty local data contribute the merge identity. The decoder uses its declared convention when the total count is zero. A round with no participating clients can be assigned a separate skip-update convention.

Associativity and commutativity permit tree or hierarchical evaluation of the same merge without requiring a physical star. Equality under reassociation uses the stipulated real-valued operations. Floating-point results may differ. Moreover, permuting clients leaves $\psi(\bigoplus_c\varphi_{n_c}(X^{(c)}))$ unchanged. A sequential client update whose output changes with client order is therefore outside the language's scope. Asynchronous shared models, updates that depend on the order of arrival, and peer states that differ across clients are not covered by this scheduling argument. Persistent private state requires the additional semantics discussed in Appendix~\ref{app:round-wise-shared-memory}.


\end{document}